%% file: icml2026.tex
\documentclass{article}

\usepackage{microtype}
\usepackage{graphicx}
\usepackage{subfigure}
\usepackage{subfiles}
\usepackage{subcaption}
\usepackage{booktabs} 
\usepackage{multirow}
\usepackage{subcaption}
\usepackage[dvipsnames]{xcolor}

\usepackage{hyperref}
\usepackage{stfloats} 

\usepackage[accepted]{icml2026}

\usepackage{amsmath}
\usepackage{amssymb}
\usepackage{mathtools}
\usepackage{amsthm}
\usepackage{bm}
\usepackage{enumitem}

\usepackage[capitalize,noabbrev]{cleveref}

\theoremstyle{plain}
\newtheorem{theorem}{Theorem}[section]
\newtheorem{proposition}[theorem]{Proposition}

\theoremstyle{definition}
\newtheorem{definition}[theorem]{Definition}

\theoremstyle{remark}

\usepackage[disable,textsize=tiny]{todonotes}

\input{symbols}
\icmltitlerunning{Universal Redundancies in TSFMs}

\usepackage{color-edits} 
\addauthor{ab}{red}
\addauthor{hv}{blue}
\addauthor{jl}{magenta}
\addauthor{wg}{green}

\begin{document}

\twocolumn[
\icmltitle{Universal Redundancies in Time Series Foundation Models}



\icmlsetsymbol{equal}{*}

\begin{icmlauthorlist}
\icmlauthor{Anthony Bao}{equal,xxx}
\icmlauthor{Venkata Hasith Vattikuti}{equal,yyy}
\icmlauthor{Jeffrey Lai}{zzz}
\icmlauthor{William Gilpin}{yyy}
\end{icmlauthorlist}

\icmlaffiliation{xxx}{ECE Department, UT Austin}
\icmlaffiliation{yyy}{Department of Physics, UT Austin}
\icmlaffiliation{zzz}{Oden Institute, UT Austin}


\icmlcorrespondingauthor{William Gilpin}{wgilpin@utexas.edu}

\icmlkeywords{Machine Learning, ICML}

\vskip 0.3in
]



\printAffiliationsAndNotice{\icmlEqualContribution} 

\begin{abstract}
Time Series Foundation Models (TSFMs) leverage extensive pretraining to accurately predict unseen time series during inference, without the need for task-specific fine-tuning. Through large-scale evaluations of standard benchmarks, we find that leading transformer-based TSFMs exhibit redundant components in their intermediate layers. We introduce a set of tools for mechanistic interpretability of TSFMs, including ablations of specific components and direct logit attribution on the residual stream. Our findings are consistent across several leading TSFMs with diverse architectures and across a diverse set of real-world and synthetic time-series datasets. We discover that all models in our study are robust to ablations of \textit{entire layers}. Furthermore, we develop a theoretical framework framing transformers as kernel regressors, motivating a purely \textit{intrinsic} strategy for ablating heads based on the stable rank of the per-head projection matrices. Using this approach, we uncover the specific heads responsible for degenerate phenomena widely observed in TSFMs, such as parroting of motifs from the context and seasonality bias. Our study sheds light on the universal properties of this emerging class of architectures for continuous-time sequence modeling.
\end{abstract}

\section{Introduction}

Following the broad success of Large Language Models (LLMs) as foundation models for Natural Language Processing (NLP), recent work has introduced foundation models for time-series forecasting \cite{pmlr-v235-goswami24a, garza2024timegpt1, das2024decoder, ansari2024chronos, woo2024moirai, liu2024moiraimoe, cohen2025timedifferentobservabilityperspective, shi2025timemoe, ansari2025chronos2univariateuniversalforecasting,lai2025panda}. Unlike \textit{local} forecast models, which must be trained on extensive historical data from a specific task, Time Series Foundation Models (TSFMs) are \textit{global models}, which produce zero-shot forecasts of Time Series unseen during training. On many data-limited tasks, TSFMs surpass local models including both classical statistical approaches (e.g. ARIMA, Exponential Smoothing) and fully-trained architectures that separately fit a set of parameters for each time-series in the target dataset \cite{aksu2024gifteval, chuck2025musedfm}. Notable architectures for \textit{local models} include N-BEATS \cite{Oreshkin2020N-BEATS:}, N-HITS \cite{challu_2023_nhits}, TFT \cite{lim2020temporalfusiontransformersinterpretable}, and PatchTST \cite{nie2023a}. In adapting the transformer architecture to time-series, TSFMs have evolved into a different class of models, with a unified set of design choices and common failure modes different from LLMs (Section \ref{subsection:design_space_failure_modes}). LLMs model sequences of discrete symbolic tokens drawn from a finite vocabulary, with explicit and compositional semantics. In contrast, TSFMs model continuous-valued temporal processes whose implicit semantics are scale-dependent, coupled to physical time, and defined by the underlying dynamics. As a consequence, TSFMs encode strong inductive biases via architectural and training choices, which we describe in Section \ref{section:tsfms}. 

Several works have investigated the direct application of LLMs to time-series forecasting. Through encoding time-series as strings of digits, \cite{gruver2023large} demonstrated zero-shot extrapolation by GPT-3 and LLaMA-2. Subsequent works \cite{zhou2023one, jin2024timellm, chang2025_llm4ts} showed promising results from finetuning pretrained LLMs on time-series datasets. Currently, TSFMs such as the Chronos family of models \cite{ansari2024chronos, ansari2025chronos2univariateuniversalforecasting}, TimesFM 2.5 \cite{das2024decoder}, Toto \cite{cohen2025timedifferentobservabilityperspective}, and Moirai \cite{woo2024moirai} represent the state-of-the-art for zero-shot time-series forecasting \cite{aksu2024gifteval}. Current directions include in-context fine-tuning \cite{faw2025incontext}, multivariate adaptation \cite{benechehab2025adapts}, and interpretability \cite{yu2025understandingimplicitbiasesdesign, yu2025understandingtransformerstimeseries, cao2025conversationaltimeseriesfoundation}. 

However, studies on interpretability and compressibility of TSFMs remain an underexplored area. With the success of TSFMs and the practical benefits of smaller models, a natural question arises: 

\textit{What mechanisms enable TSFMs to construct effective zero-shot forecasts?}

\begin{figure*}[!t]
    \centering
    \includegraphics[width=0.86\linewidth]{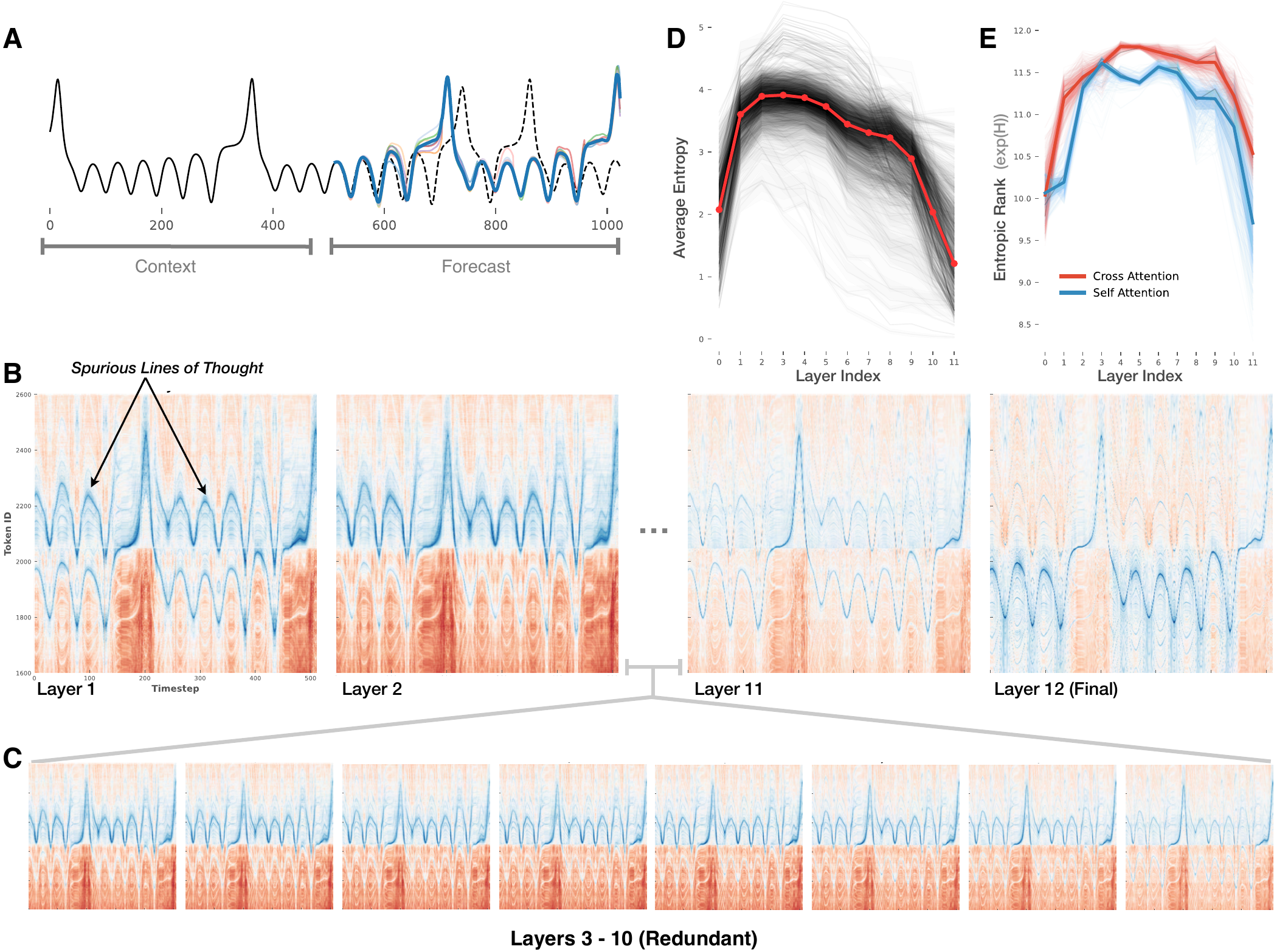}
    \caption{\textbf{Middle layers are highly redundant in their contribution to the residual stream:} For an example forecasting task \textbf{(A)}, we investigate the logit maps produced by direct logit attribution (DLA) on the residual stream \textbf{(B)} after each layer in the \textit{Chronos} decoder. As seen in \textbf{(C)}, the middle layers qualitatively perform very similar updates. We quantify this observation \textbf{(D)} by showing that these middle layers introduce \textit{uncertainty}, measured as an increase in the entropy of the resulting distribution over tokens. Specifically, we compute $\frac{1}{T}\sum_{t=0}^{T} \sigma(H^{(\ell)}W_{\text{out}})$ for $401$ distinct forecast tasks; we highlight the median for each layer. We also measure the similarity among head outputs for the heads in each layer, and we observe a higher average entropic rank (Appendix \ref{section:entropic_rank_discussion}) in the middle layers \textbf{(E)}, further suggesting redundancy in the middle layers. Appendix \ref{section:more_residual_stream_measurements} presents more examples of measurements on the residual stream.}
    \label{fig:residual_stream_logit_maps}
\end{figure*}

\newpage
Our main contributions are:
\begin{enumerate}
    \item We show that a wide variety of modern TSFMs contain highly-redundant internal components, allowing us to ablate 28\% of the total number of attention heads (and even some MLPs) for a 6\% or smaller drop in performance (Table \ref{tab:ablations_by_the_numbers_main_text}), across leading TSFMs on GIFT-Eval, a standard benchmark for zero-shot forecasting.
    \item We introduce a toolbox for mechanistic interpretability of TSFMs, and use it to study the role of layer depth. We ablate all components in each layer to demonstrate that the middle layers are highly redundant (Section \ref{subsection:ablations_entire_layer}). This contrasts with the crucial representation-building role of early and late layers. To understand this effect, we zoom in on the residual stream of a single model, \textit{Chronos}, and find that middle layers often perform redundant functions, with miniscule effect on the final prediction. Furthermore, initial layers can introduce spurious lines of thought visible in the logit maps on autoregressive rollout (Section \ref{subsection:measurements_on_residual_stream} and Fig. \ref{fig:residual_stream_logit_maps}B), which later layers must resolve.
    \item We create a theoretical kernel-based model of temporal attention for time series, and use it to motivate a highly-effective \textit{intrinsic} strategy for ablating individual attention heads within each layer, based on the stable rank of the query-key projection matrix (Section \ref{section:srank_ablations}). Through head-level ablation experiments, we show that the most ablateable layers require the fewest number of heads to maintain the model's performance. 
    \item Across multiple TSFMs, we identify specific attention heads responsible for degenerate phenomenon that undermine the generalization of TSFMs: context parroting, in which models repeat exact sequences from their context, and bias towards particular seasonality in produced forecasts. We study this finding through two perspectives: via a direct measurement of the attention rollout (Fig. \ref{fig:head_sharpness}), and using our \textit{intrinsic} head-level ablation strategy (Section \ref{section:srank_ablations}).
\end{enumerate}

Besides interpretability, we believe our study \footnote{Code Available: \href{https://github.com/abao1999/tsfm-lens}{github.com/abao1999/tsfm-lens}} has important implications for compression of TSFMs. We discuss these implications in Section \ref{section:conclusion}. We review the extensive body of related work in Section \ref{section:related_work}.

\subfile{sections/main_text/section_tsfms_and_residual_stream}
\newpage
\subfile{sections/main_text/section_ablations_entire_layers}
\subfile{sections/main_text/section_kernel_regression}

\subfile{sections/main_text/section_ablations_heads}
\subfile{sections/main_text/section_dataset_view}
\subfile{sections/main_text/section_related_work}

\section{Conclusion} \label{section:conclusion}
Our investigation uncovers the presence of redundant components in the intermediate layers of several leading time series foundation models (TSFMs) of diverse architectures. Our theoretical kernel regression framework provides insight into the development of phenomena such as context parroting and seasonality bias commonly observed in TSFMs. Motivated by these insights, we also empirically establish the effectiveness of a purely intrinsic (data-independent) method of pruning attention heads, and we demonstrate that the essential structure of the forecasts is preserved even under heavy ablations of the heads. We develop and open-source a toolbox for mechanistic interpretability of TSFMs, which we utilize to set up ablations of specific components, and to examine the residual stream. We believe our study motivates future work in model compression and understanding the mechanisms underlying the effectiveness of zero-shot forecasting.

\section*{Impact Statement}
This paper presents work whose goal is to advance the field of Machine Learning. There are many potential societal consequences of our work, none which we feel must be specifically highlighted here.

\section*{Acknowledgements and Disclosure of Funding}
The authors thank Jerry Liu and Yasa Baig for insightful discussion and encouraging feedback. AB was supported by the UT PGEF Fellowship
and the Basdall Gardner Memorial Fellowship. VHV was supported by the Moncrief Summer Internship, the College of Natural Sciences Scholarship TX26, and Melvin J. Rieger Scholarship Fund in Physics. JL was supported by the UT CSEM Fellowship. WG was supported by NSF DMS 2436233 and NSF
CMMI 2440490. This project has been made possible in part by Grant No. DAF2023-329596 from
the Chan Zuckerberg Initiative DAF, an advised fund of Silicon Valley Community Foundation. 
Financial support for this publication results from grant CS-CSA-2026-075 from Research Corporation for Science Advancement.
The authors acknowledge the Biomedical Research Computing Facility and Texas Advanced Computing Center (TACC) at The University of Texas at Austin for providing computational resources.

\bibliography{refs}
\bibliographystyle{icml2026}

\newpage
\appendix
\onecolumn

\subfile{sections/appendix/appendix_design_space_tsfms}
\subfile{sections/appendix/appendix_zero_ablation}
\newpage
\subfile{sections/appendix/appendix_more_ablation_experiments}
\newpage
\subfile{sections/appendix/appendix_spectral_sharpness}
\newpage
\subfile{sections/appendix/appendix_extended_kernel_discussion}
\subfile{sections/appendix/appendix_entropic_rank}
\newpage
\subfile{sections/appendix/appendix_additional_measurements_residual_stream}
\newpage
\subfile{sections/appendix/appendix_induction_heads}
\end{document}

%% file: symbols.tex
\newcommand{\bSigma}{\bm{\Sigma}}

\newcommand{\bA}{\mathbf{A}}

\newcommand{\bE}{\mathbf{E}}

\newcommand{\bM}{\mathbf{M}}

\newcommand{\bS}{\mathbf{S}}

\newcommand{\bU}{\mathbf{U}}
\newcommand{\bV}{\mathbf{V}}
\newcommand{\bW}{\mathbf{W}}

\newcommand{\bh}{\mathbf{h}}

\newcommand{\bk}{\mathbf{k}}

\newcommand{\bq}{\mathbf{q}}

\newcommand{\bs}{\mathbf{s}}

\newcommand{\bv}{\mathbf{v}}

\newcommand{\bbR}{\mathbb{R}}

\newcommand{\norm}[1]{\left\|#1\right\|}

\newcommand{\dmodel}{d_{\mathrm{model}}}
\newcommand{\dhead}{d_{\mathrm{head}}}

\newcommand{\hq}{\mathbf{h}_{q}}

\theoremstyle{plain}

\newcommand{\disti}{\norm{\bq-\bk_i}^2}
\newcommand{\distj}{\norm{\bq-\bk_j}^2}

%% file: sections/main_text/section_tsfms_and_residual_stream.tex
\section{Time Series Foundation Models} \label{section:tsfms}
We investigate a representative collection of recent transformer-based time series foundation models (TSFMs) based on their diversity in architecture and design choices, in addition to their performance on GIFT-Eval \cite{aksu2024gifteval}, a leading benchmark for zero-shot forecasting ability. The models in our study are: \textit{Chronos} \cite{ansari2024chronos}, \textit{Chronos-Bolt}, \textit{Chronos 2} \cite{ansari2025chronos2univariateuniversalforecasting}, \textit{TimesFM 2.5} \cite{das2024decoder}, \textit{Toto} \cite{cohen2025timedifferentobservabilityperspective}, and \textit{Moirai 1.1} \cite{woo2024moirai}.

\subsection{Design Space and Failure Modes} \label{subsection:design_space_failure_modes}
In Table \ref{tab:tsfm_comparison} we review the architectures, tokenization schemes, and other key design choices of the TSFMs in our study. Appendix \ref{section:design_space_tsfms_appendix} presents further details.

Modern TSFMs rely on architectural and training choices to encode strong inductive biases such as temporal locality, seasonality, and scale invariance. Patching enforces local temporal continuity by grouping contiguous observations, biasing the model toward stable short-range patterns, while multi-scale patching supports abstraction across timescales \cite{yu2025understanding,lai2025panda}. Scale invariance is crucial for generalization across context windows with differing absolute magnitudes and units, and is commonly enforced through instance normalization and quantile forecasting. The arrow of time (i.e. the future cannot influence the past) is often encoded via causal masking and autoregressive or decoder-based architectures \cite{nie2023a,zhang2025autohformer}. These inductive bias lead to common failure modes of time series foundation models, which recur across distinct architectures. In the case of \textit{context parroting}, models tend to repeat long sequences directly from their context \cite{zhang2024zero,zhang2025contextparrotingsimpletoughtobeat}. This occurs particularly in LLM-inspired TSFM architectures, and it often results in unexpectedly strong zero-shot performance in stationary settings---when history repeats itself. Regression-based models, however, tend to exhibit seasonality bias, in which the model preferentially forecasts single frequencies, not due to bias in the training data, but rather interaction of the model's time-axis hyperparameters (e.g. patch size) with the context \cite{zhang2024multi,yu2025understanding,zhou2022fedformer,wu2021autoformer,zeng2023transformers}. A special case of seasonality bias is mean regression, a common phenomenon in which TSFMs predict a context-dependent constant mean value \cite{zhou2025transformers,lai2025panda,yu2025understanding}.

\subsection{Residual Stream}  \label{section:residual_stream}
The central object of our study is the residual stream. In Appendix \ref{section:residual_stream_appendix}, we outline the residual stream updates for each of the models considered in our investigation. While the specific architectures and design spaces of these models vary considerably, they share a unified residual stream update strategy, and we provide a high-level abstraction (Equation \ref{eq:residual_stream}). All models that we consider have pre-norm architectures. Let $H^{(\ell)}$ for $\ell \in 0, \ldots, L$ denote the residual stream (sequence of hidden states) at layer $\ell$. And let $X_{\mathrm{context}}$ denote the context input to the model. We denote multi-head self-attention blocks by \textit{SA}. In models with an encoder and decoder architecture (\textit{Chronos} and \textit{Chronos-Bolt}), we denote the multi-head cross-attention blocks by \textit{CA}. In \textit{CA} blocks, the decoder attends to the encoder outputs $\mathcal{E}$, using its own hidden state as the queries and the linear projections of $\mathcal{E}$ as the keys and values. We denote the norm applied (i.e. LayerNorm, InstanceNorm, RMSNorm, etc) by $N(\cdot)$.
\[
H^{(0)} = \text{Embed}\,(X_{\mathrm{context}}) \in \mathbb{R}^{N \times d_{\mathrm{model}}}    
\]
\begin{align*}
    \begin{split}
    H^{(\ell)} = H^{(\ell-1)} 
          &+ \underbrace{\mathrm{SA}\,\!\big(N(H^{(\ell-1)})\big)}_{H^{(\ell)}_{\mathrm{SA}}}
          + \underbrace{\mathrm{CA}\,\!\big(N(H_{\mathrm{SA}}^{(\ell)}),\mathcal{E}\big)}_{H^{(\ell)}_{\mathrm{CA}}} \\
          &+ \underbrace{\mathrm{MLP}\,\!\big(N(H^{(\ell)}_{\mathrm{CA}})\big)}_{H^{(\ell)}_{\mathrm{MLP}}}
    \end{split}
\end{align*} \label{eq:residual_stream}

And $W_{\mathrm{out}}$ denotes the final transformation that maps the residual stream $d_{\mathrm{model}}$ to either logits over the vocabulary (in the case of \textit{Chronos}), or directly to the forecast horizon $X_{\mathrm{pred}} = N(H^{(L)}) \, W_{\mathrm{out}}$.

\subsection{Measurements on the Residual Stream} \label{subsection:measurements_on_residual_stream}
In Fig. \ref{fig:residual_stream_logit_maps}, we visualize the logit maps after each layer for an example forecast by \textit{Chronos}, on autoregressive rollout. As we further illustrate in Appendix \ref{section:more_residual_stream_measurements}, the initial layers introduce what we term \textit{spurious lines of thought}: trajectories in the logit map rollouts that are often infeasible (e.g. not continuous with the context) and which are later washed out by deeper layers. We observe that the middle layers take on largely redundant roles, and do not impose a significant effect on the residual stream. After each middle layer, the logit distribution exhibits a high degree of \textit{uncertainty} i.e. indecision among several possible lines of thought (high values for a wider spread of tokens at each time step). The distribution thus has higher entropy (duller peaks), which we quantify in Fig. \ref{fig:residual_stream_logit_maps}D. In contrast, the last layer consolidates the representation by \textit{selecting} the final prediction, the essential structure of which has existed in the residual stream since the first layer.

%% file: sections/main_text/section_ablations_entire_layers.tex
\section{Ablating Entire Layers} \label{section:ablations}

Across all of our ablation experiments (see Appendix \ref{section:ablations_methodology_appendix} for details on ablation methodology), including in Section \ref{section:srank_ablations}, we use GIFT-Eval \cite{aksu2024gifteval}, a standard benchmark for evaluating TSFMs on zero-shot time-series forecasting. The GIFT-Eval test set includes 144,000 time series across 7 domains, 10 frequencies, and prediction lengths from short to medium to long-term.

\begin{figure*}[t]
    \centering
    \includegraphics[width=1.0\linewidth]{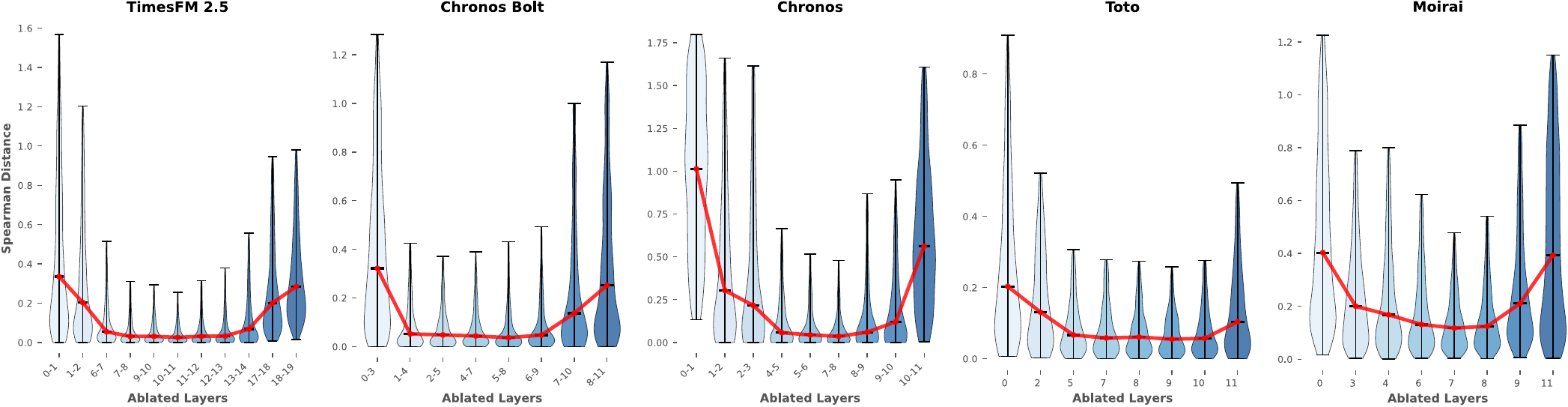}
    \caption{\textbf{Middle Layers are more ablateable than early and late layers:} Spearman distance ($1 - \rho$) between the original model predictions and the predictions with \textit{ablations of entire layers} (i.e. All Heads and the MLP) for groups of layers. Across all the TSFMs we evaluate, the middle layers show greater ablatability, suggesting redundant components, whereas the first and last layers are the most important to preserving the model's performance. Appendix \ref{section:more_ablation_results} presents more measurements of the depthwise importance of components.}
    \label{fig:ablations_spearman}
\end{figure*}

\label{subsection:ablations_entire_layer}
We zero-ablate the contribution to the residual stream of several components within the transformer, in order to elucidate the role of layer depth.

\begin{figure}[htbp]
    \centering
    \includegraphics[width=0.9\linewidth]{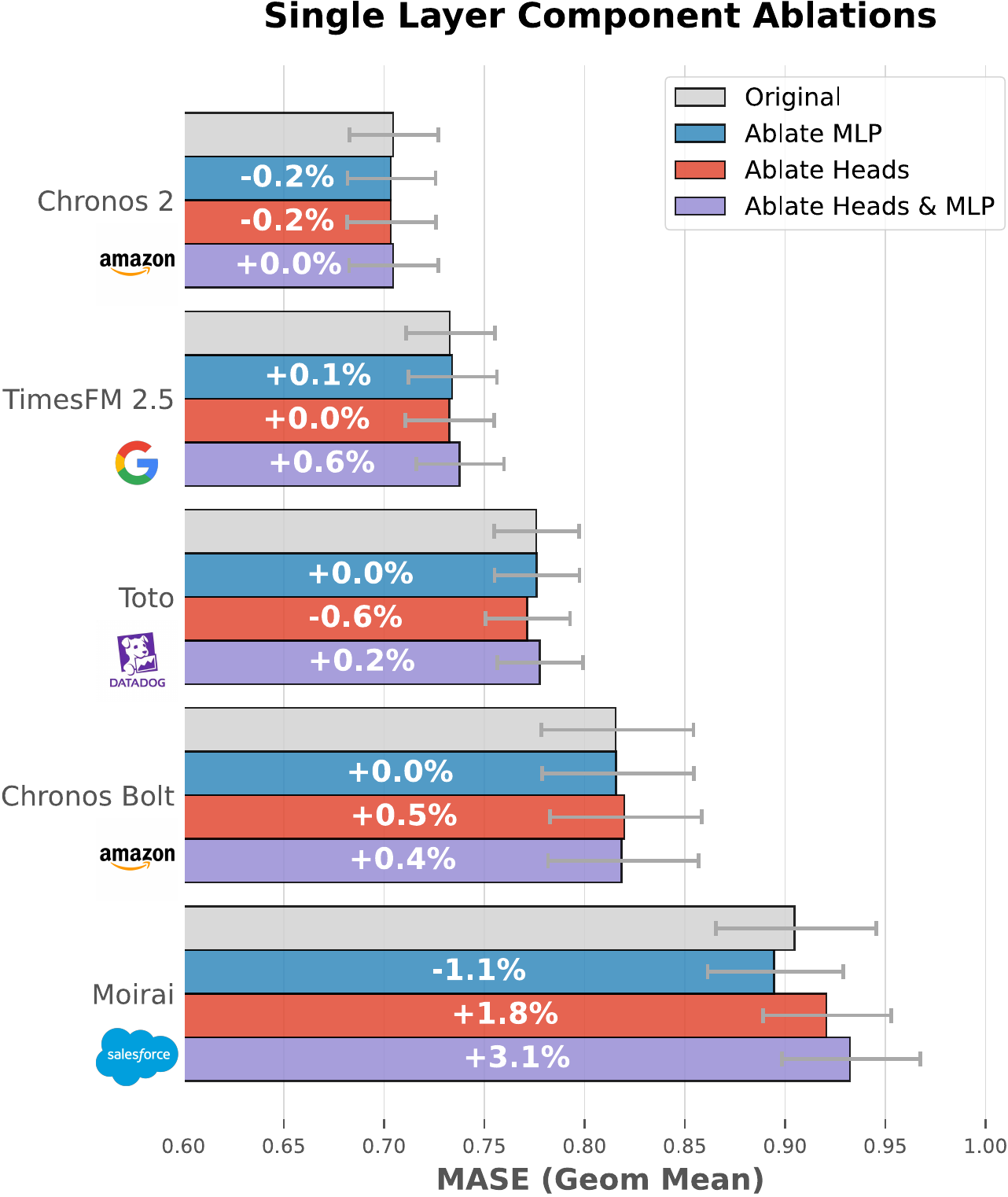}
    \caption{\textbf{TSFMs have redundant components}: Across all the leading models in our study, we identify layers with redundant components. We report the MASE geometric mean of the models with ablations against the original model on GIFT-Eval.}
    \label{fig:single_layer_component_ablations}
\end{figure} 

In Fig. \ref{fig:single_layer_component_ablations} we present the result of ablating, for a single layer: 1) the MLP; 2) \textit{all} attention heads (self-attention and also cross-attention for the models that have it); and 3) all heads and the MLP. Specifically, we perform the ablations on the most ablateable layers. We identify these layers through ablation experiments on all components, for each layer of each model (Figs. \ref{fig:ablations_spearman}, \ref{fig:ablations_spearman_combined_heads}, \ref{fig:ablations_spearman_combined_mlps}).

\begin{figure}[htbp]
    \centering
    \includegraphics[width=0.85\linewidth]{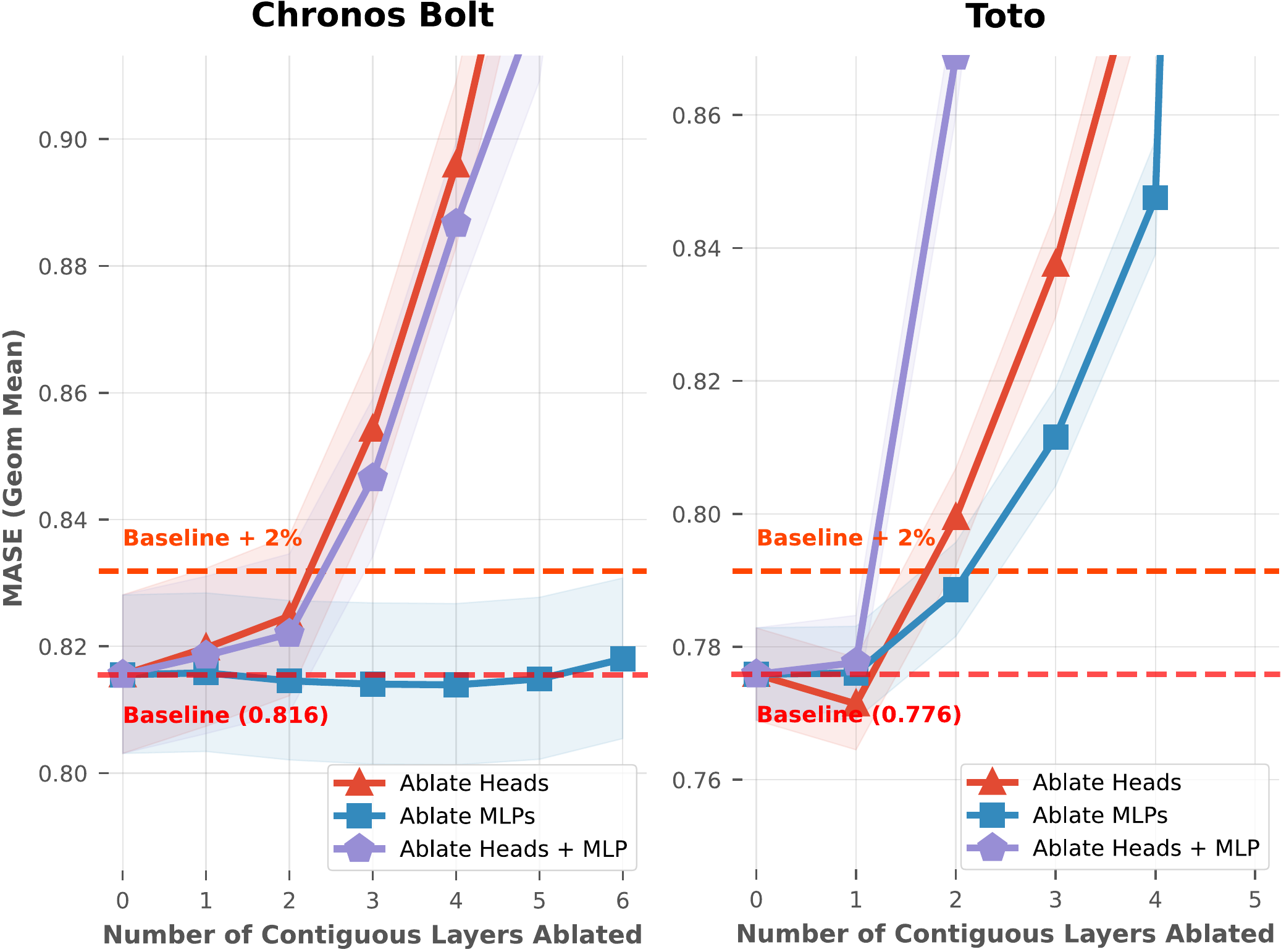}
    \caption{\textbf{TSFMs exhibit significant variation in sensitivity to component ablations:} When ablating components in contiguous layers, \textbf{(A)} Chronos Bolt shows extreme resilience to the MLP ablations, maintaining its performance even after half of them are removed. \textbf{(B)} Toto shows much greater sensitivity.}
    \label{fig:components_comparison_bolt_toto}
\end{figure}

We observe a significant difference in sensitivity to MLP ablations across architectures (Fig. \ref{fig:components_comparison_bolt_toto}). We find that \textit{Chronos Bolt}, an encoder-decoder architecture, is much more resilient to MLP ablations than decoder-only architectures such as \textit{Toto} and \textit{TimesFM 2.5} (Appendix \ref{section:more_ablation_results}). This motivates future work: \textit{how does architecture determine what the transformer loads onto MLPs versus attention heads?}

In Appendix \ref{section:more_details_tsfms_appendix} we detail the settings e.g. max context length, number of samples, patch length, we use for each TSFM in the evaluation. For runtime considerations, some settings differ from those used in GIFT-Eval leaderboard entries.

%% file: sections/main_text/section_kernel_regression.tex
\begin{figure*}[t]
    \centering
    \includegraphics[width=1.0\linewidth]{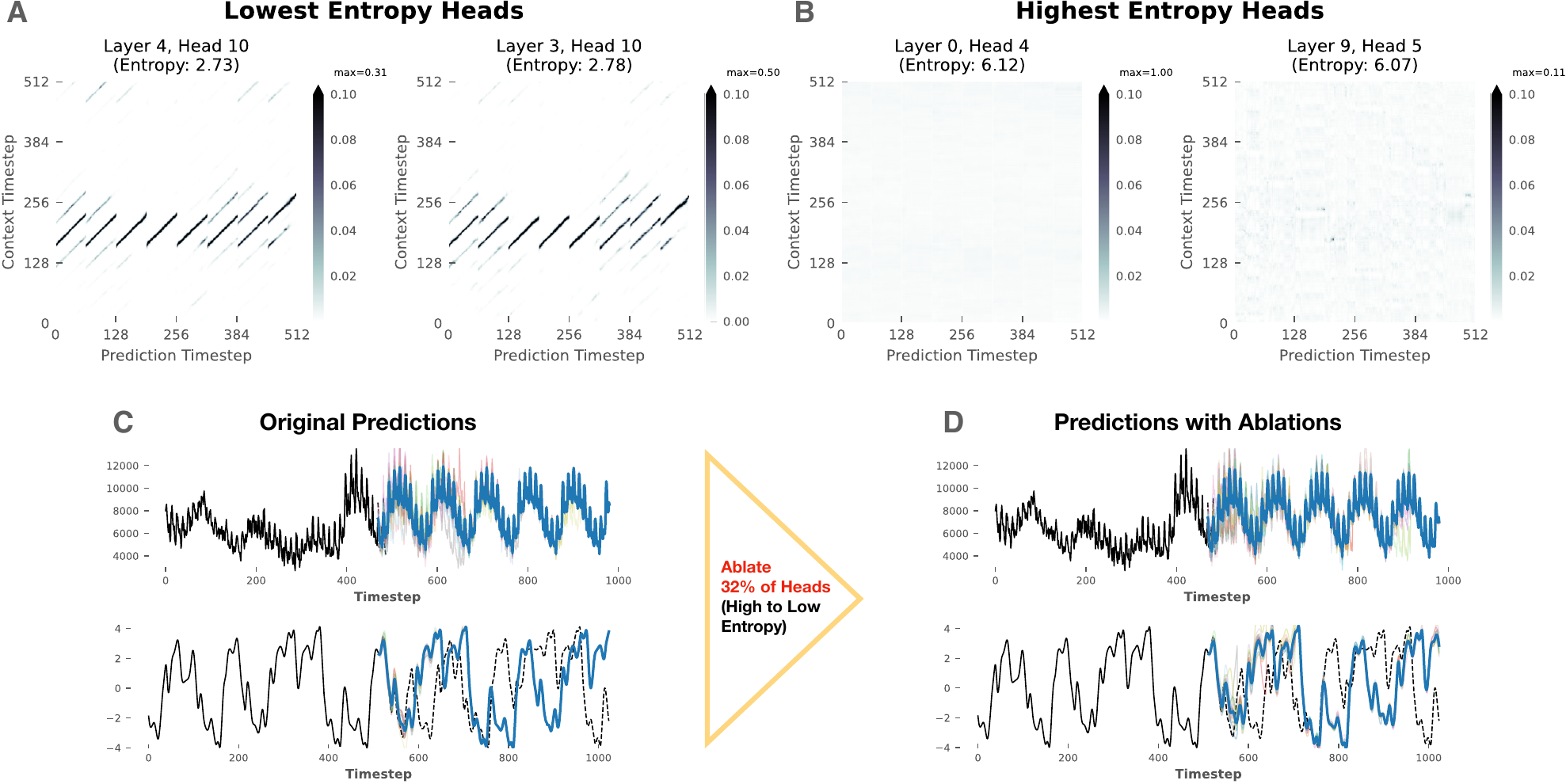}
    \caption{\textbf{Sharp heads as a mechanism for context parroting:} Attention scores between prediction and context timesteps for \textit{Chronos} show clear delineation of low entropy``sharp" \textbf{(A)} and high entropy ``diffuse" heads \textbf{(B)}. Sharp heads correspond to kernel regression with a small bandwidth, whereas diffuse heads are highly redundant and can ablated without affecting the context parroting \textbf{(C)} and \textbf{(D)}. Moreover, removing a single sharp head can break the parroting (Fig. \ref{fig:chronos_failure_ablate_sharp_heads}). Appendix \ref{section:extended_discussion_attention_kernel_regression} presents further discussion.}
    \label{fig:head_sharpness}
\end{figure*}

\section{Attention Heads as Kernel Regression} \label{section:heads_as_kernel_regression}

We find that time series offer an interesting substrate for studying transformer behavior. We focus on the well-known correspondence of the attention mechanism as a Nadaraya-Watson (NW) estimator \cite{tsai2019transformerdissectionunifiedunderstanding,wang2025testtimeregressionunifyingframework}. In Appendix \ref{section:extended_discussion_attention_kernel_regression} we present an extended discussion of this kernel regression correspondence.

For a cross attention layer in an encoder-decoder style transformer, let $\bh_i \in \bbR^{\dmodel}$ denote the encoder hidden states for the context tokens ($i \in [C]$) and  $\hq \in \bbR^{\dmodel}$  the hidden state for a query token in the decoder residual stream. For any layer, denote the projection matrices for the queries and as $\bW_Q, \bW_K, \bW_V \in \bbR^{\dmodel \times \dhead}$ respectively. Let $\bq = \hq^\top \bW_Q$, $\bk_i = \bh_i^\top \bW_K$, $\bv_i = \bh_i^\top \bW_V$ be the query, key, and value vectors where the key and value vectors are indexed from the encoder context. The cross-attention output is $\sum_{i=1}^C w_i \bv_i$ where $w_i = \text{Softmax}(\bs)_i$ is the softmax (SM) over the attention scores $\bs_i = \langle\bq, \bk_i \rangle / \sqrt{d_\text{head}}$.

Noting that $\langle \bq, \bk_i \rangle = \frac{1}{2} (\norm{\bq}^2 + \norm{\bk_i}^2 - \disti)$ and assuming that the keys have the same norm, the SM yields:
\begin{equation}
\label{eq:approxattn}
w_i = \frac{\exp{(-\disti/2\tau)}}{\sum_{j=1}^{C} \exp{(-\distj /2\tau)}}
\end{equation}
where $\tau = \sqrt{\dhead}$. This corresponds to a Gaussian NW estimator with bandwidth $\tau$ \cite{wang2025testtimeregressionunifyingframework}. However, without a pre-norm layer (e.g. QK-norm \cite{henry2020query}), the constant norm assumption is unrealistic. Instead, we provide an interpretable spectral characterization for the NW kernel in SM attention.

\subsection{Spectral Sharpness in Attention Heads}
\label{subsection:spectral}

Let $\bM \coloneqq \bW_Q \bW_K^\top / \sqrt{\dhead} = \bU\bS\bV^\top\in \bbR^{\dhead \times \dhead}$ be the SVD of the combined query and key projection matrices for a head. Then, $\langle \bq, \bk_i\rangle / \sqrt{\dhead} = \langle \hq , \bM \bh_i \rangle = \langle \bU^\top \hq, \bS \bV^\top \bh_i \rangle $. Let $\tilde{\bh}_q \coloneqq \bU^\top \hq$, $\tilde{\bh}_i \coloneqq \bV^\top \bh_i$, then $\langle \tilde{\bh}_q, \bS\tilde{\bh}_i \rangle = \frac{1}{2} ( ||\tilde{\bh}_q||^2_\bS + ||\tilde{\bh}_i||^2_\bS - ||\tilde{\bh}_q - \tilde{\bh}_i||^2_\bS)$, where $||\bv||^2_\bS = \bv^\top \bS \bv$. The SM yields the attention weights:
\[
w_i \sim \exp\left(\frac{1}{2}||\tilde{\bh}_i||^2_\bS\right)\exp\left(-\frac{1}{2}(\tilde{\bh}_q -\tilde{\bh}_i)^\top \bS\,(\tilde{\bh}_q -\tilde{\bh}_i)\right)
\]
The second factor can be interpreted as a multivariate Gaussian kernel with diagonal covariance $\Sigma = \bS^{-1}$. The first factor tilts the Gaussian kernel with the $\bS$-seminorm of the $\bV$-projected keys. Ignoring the tilt, notice that each direction of the query-key difference $(\tilde{\bh}_q - \tilde{\bh}_i)_j$ independently follows a univariate Gaussian kernel with bandwidth $1/\sqrt{\sigma_j}$ where $\sigma_j$ is the $j$-th singular value of $\bM$. Thus, sharpness can be expected when $\sqrt{\sigma_j} \gg (\tilde{\bh}_q - \tilde{\bh}_i)_j$ since this corresponds to a narrow bandwidth. Moreover, the tilt factor can also contribute to the sharpness if a $\bV$-projected key strongly aligns with the principal axes of $\bM$. Even slightly aligned directions on the basis $\bV$ can sharpen the weights if a few singular values are extremely large.

\begin{figure*}[t]
    \centering
    \includegraphics[width=1.0\linewidth]{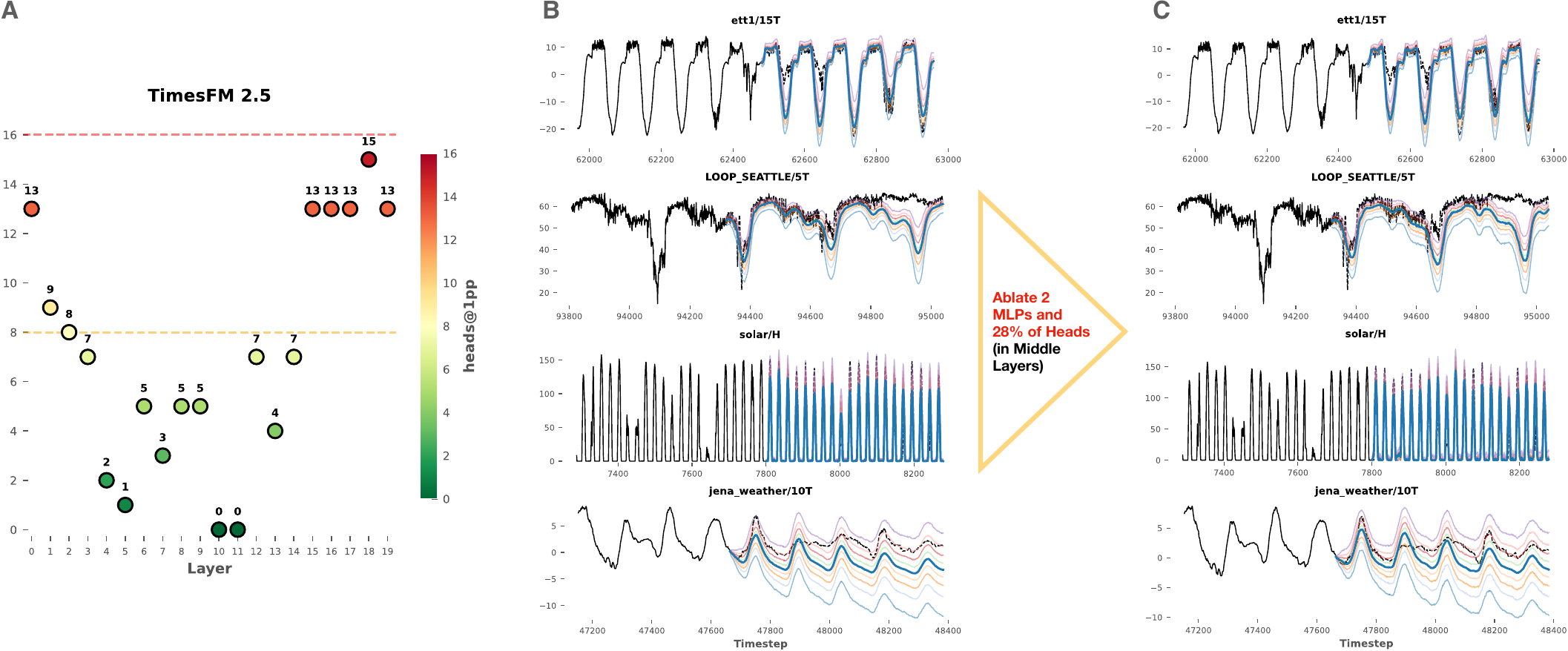}
    \caption{\textbf{Layerwise ablation of heads by stable rank preserves the forecasts:} (A) Using our \texttt{heads@1pp} metric, we find that the most ablateable layers (Fig. \ref{fig:ablations_spearman}) require the fewest number of heads to maintain the model's performance. (B) We present example forecasts using TimesFM 2.5 on long-term forecasting tasks of various timescales, from the GIFT-Eval test set. (C) We zero-ablate the middle layers (7-13 inclusive) down to \texttt{heads@1pp} found using our stable rank strategy, which corresponds to 28\% of all heads. Moreover, we zero-ablate the MLPs of layers 10 and 11 to entirely remove the contribution of those layers. Note that all forecasts presented use context length 2048, consistent with our TimesFM 2.5 evaluation setting (Appendix \ref{section:more_details_tsfms_appendix}), but we plot only the last 512 time points of context. Appendix \ref{subsection:preserving_essential_structure} shows the pitfalls of random head selection, and that our intrinsic ablation strategy preserves the essential seasonality structure of the forecasts, across all models, and even when ablating heads from all layers of TimesFM 2.5 (51\% of heads, Fig. \ref{fig:timesfm2p5_ablate_alllayers_seasonality_versus_random}).}
    \label{fig:timesfm2p5_srank_combined}
\end{figure*}

Based on the intuition that diffuse cross-attention heads produce conservative forecasts such as mean regression, we hypothesize that sharp (low-entropy) heads are more important for forecasting ability. Fig. \ref{fig:chronos_failure_ablate_sharp_heads} shows that ablating even a single sharp head causes significant degradation in the forecasts. Therefore, the spectral interpretation naturally motivates a model compression strategy: simply prune the attention heads with low spectral energy. Prior work has shown that simple ``parroting" heads are a surprisingly effective driver of long term forecasting performance in seasonal and even chaotic dynamics \cite{zhang2025contextparrotingsimpletoughtobeat}.

We now prove that the sharpness in the softmax attention matrix is directly controlled by the spectrum of the head under mild alignment assumptions:
\begin{proposition}[Informal]
\label{prop:concentration}
    If each decoder query $\bh_{q_i}$ uniquely best-aligns with an encoder key $\bh_{k_j}$ with positive margin, then SM attention concentrates exponentially around $\bh_{k_j}$ with rate controlled by the top-$r$ spectral modes. 
\end{proposition}
The result is formally restated and proved in Appendix \ref{app:proof}. Consequently, it is clear that if contiguous subsequences in the context and decoder queries are aligned at a fixed lag $\tau$, then the attention matrix will exhibit a diagonal band (local shift) structure which appears in \textit{Chronos} in Fig. \ref{fig:head_sharpness}A. Beyond \textit{Chronos}, this local shift matrix structure is even demonstrated in the cross-attention matrices over patches of different sizes in Chronos Bolt (Fig. \ref{fig:bolt_parrot}, Appendix \ref{app:patchedmodels}). Although the result is technically only applicable to encoder-decoder style models, we show that the compression strategy motivated by this theory is effective across model classes.

%% file: sections/main_text/section_ablations_heads.tex
\section{Ablating Heads by Stable Rank} \label{section:srank_ablations}
In Section \ref{subsection:ablations_entire_layer} we included ablations of \textit{all} attention heads in each layer. We now take a more fine-grained view and conduct head-level ablations for each layer motivated by the results in Section \ref{subsection:spectral}. In Fig. \ref{fig:timesfm2p5_srank_combined} and the sequel, we illustrate the surprising effectiveness of an \textit{intrinsic} head pruning strategy that orders the heads in each layer by the stable rank of the query-key projection matrix i.e. $\operatorname{sr}(W_Q W_K^{\top})$.

\begin{definition}[\textit{Stable Rank}]
The stable rank of a matrix $A$ is the ratio of its Frobenius norm and its spectral norm.
\[
\operatorname{sr}(A) = \|A\|_F^2/\|A\|_2^2
\]
where $\|A\|_F^2 := \sum_{i=1}^{r}\sigma_i(A)^2$ and $\|A\|_2^2 := \sigma_1(A)^2$.
\end{definition}

\begin{definition}[\texttt{heads@1pp}] \label{definition:heads@1pp}
Consider a head-pruning strategy that induces an ordering $\{h_i\}_{i=0}^{N_{\text{heads}}}$ of attention heads, in layer $\ell$ of a multi-head attention transformer. 

Let $\mathcal{M}(S)$ denote the model after zero-ablating the set of heads $S \subseteq \{h_i\}_{i=0}^{N_{\text{heads}}}$, and let $\mathcal{E}(\cdot)$ denote a scalar error metric (e.g. MASE).

For $k \in \{0,1,\dots,N_{\text{heads}}\}$, denote the set of heads that are \textit{kept} as $\mathcal{K}_k := \{h_1,\dots,h_k\}$ and the complementary set of $N_{\text{heads}} - k$ heads that are \textit{ablated} as $\mathcal{A}_k := \{h_i\}_{i=1}^{N_{\text{heads}}} \setminus \mathcal{K}_k$.
Define \texttt{heads@1pp} to be the smallest integer $k$ such that
\[
\frac{\mathcal{E}\bigl(\mathcal{M}(\mathcal{A}_k)\bigr) - \mathcal{E}(\mathcal{M}(\varnothing))}{\mathcal{E}(\mathcal{M}(\varnothing))}
\;\le\; 1\%.
\]
In other words, \texttt{heads@1pp} for a layer is the minimal number heads that must be \emph{kept}, following the ordering, with all other heads zero-ablated, to maintain the ablated model's error within one percentage point of the original.
\end{definition}

The goal of our head-level ablation strategy is to minimize \texttt{heads@1pp} (i.e. maximize the number of ablated heads) for each layer. We observe that many highly ablateable layers achieve lower \texttt{heads@1pp} when ablating heads from high to low stable rank (Fig. \ref{fig:ablate_srank_high_low_timesfm2p5}). Heads with high stable rank may correspond to diffuse heads that perform weakly structured global smoothing. Meanwhile, heads with low stable rank can be more specialized if they align with a few task-relevant modes that encourage seasonality, potentially producing sharp attention patterns. We provide further discussion in Section \ref{subsection:implications_srank}.

Table \ref{tab:ablations_by_the_numbers_main_text} presents several heavy ablations. Appendix \ref{subsection:preserving_essential_structure} presents example forecasts that demonstrate the preservation of the essential seasonality structure of the forecasts. We emphasize that our aim is \textit{not} to propose an optimal pruning strategy, but rather to demonstrate the surprising effectiveness of an \textit{intrinsic} (data-independent) strategy for pruning attention heads within layers.

\begin{table}[htbp]
\centering
\scalebox{0.86}{%
\begin{tabular}{lcc}
\toprule
\multicolumn{3}{c}{\textbf{Evaluation of Ablated Models on GIFT-Eval test set}} \\
\midrule
Model \footnotesize{(ZA $\textcolor{WildStrawberry}{\% \text{Heads}} + \textcolor{Periwinkle}{N_{\text{MLP}}}$)}
& \footnotesize{$\% \Delta$ MASE $(\downarrow)$}
& \footnotesize{$\% \Delta$ CRPS $(\downarrow)$} \\
\midrule
\textbf{Chronos 2 \footnotesize{(\textcolor{WildStrawberry}{24\%} + \textcolor{Periwinkle}{2 MLP})}} &
\textbf{4.71\%} & \textbf{3.92\%} \\
\midrule
\textbf{TimesFM 2.5 \footnotesize{(\textcolor{WildStrawberry}{28\%} + \textcolor{Periwinkle}{2 MLP})}} &
\textbf{6.12\%} & \textbf{5.86\%} \\
\midrule
\textbf{Toto \footnotesize{(\textcolor{WildStrawberry}{29\%})}} &
\textbf{3.54\%} & \textbf{3.60\%} \\
Toto \footnotesize{(\textcolor{WildStrawberry}{29\%} + \textcolor{Periwinkle}{1 MLP})} &
8.54\% & 8.73\% \\ 
Toto \footnotesize{(\textcolor{WildStrawberry}{43\%})} &
10.10\% & 9.97\% \\
\midrule
\textbf{Chronos Bolt \footnotesize{(\textcolor{WildStrawberry}{28\%} + \textcolor{Periwinkle}{6 MLP})}} &
\textbf{4.24\%} & \textbf{0.19\%} \\
Chronos Bolt \footnotesize{(\textcolor{WildStrawberry}{42\%} + \textcolor{Periwinkle}{6 MLP})} &
10.71\% & 6.42\% \\
\bottomrule
\end{tabular}
}
\vspace{6pt}
\caption{We evaluate leading TSFMs under heavy ablations of heads found via our intrinsic head-pruning strategy (Section \ref{section:srank_ablations}), in addition to some MLPs. We report the percent change (lower is better) in MASE and CRPS geometric mean on GIFT-Eval. We \textbf{bold} the ablations for which we simply take $\approx 1/3$ of all layers (identified to be most ablateable in Appendix \ref{section:more_ablation_results}) to \texttt{heads@1pp}. Table \ref{tab:ablations_by_the_numbers} presents additional metrics.}
\label{tab:ablations_by_the_numbers_main_text}
\end{table}

\subsection{Implication of Stable Rank} \label{subsection:implications_srank}

We argue that TSFMs represent a wholly different setting from LLMs. As we elucidated in Section \ref{section:heads_as_kernel_regression}, the attention heads in TSFMs approximate kernel operators. We hypothesize that the performance of TSFMs on forecasting tasks is dominated by similarity comparisons i.e. pattern matching over time. From this perspective, the intrinsic properties of attention heads can be informative about functional capacity. 

In classical kernel learning theory, the effective dimension captures how many modes are active after regularization i.e. how many directions the learned kernel can meaningfully exploit, thus bounding the complexity of the learned estimator. Meanwhile, the stable rank is a purely spectral quantity that captures the spread of spectral energy relative to the largest mode. It provides a coarse and scale-free (i.e. invariant to scaling the kernel by a positive scalar) notion of effective dimensionality \cite{vershynin2018high}.

A high stable rank indicates that many spectral modes contribute to that head's attention score, whereas a low stable rank indicates that a small number of modes dominate. A head with stable rank $r = \operatorname{sr}(W_Q W_K^{\top})$ has its spectral mass approximately concentrated in $r$ comparable directions. Thus, under some data smoothness assumptions, the heads with low stable rank $r$ effectively behave as low-dimensional operators, in the sense that their action on the data is well-approximated by a small set of dominant modes. 

\begin{figure}[t]
    \centering
    \includegraphics[width=0.75\linewidth]{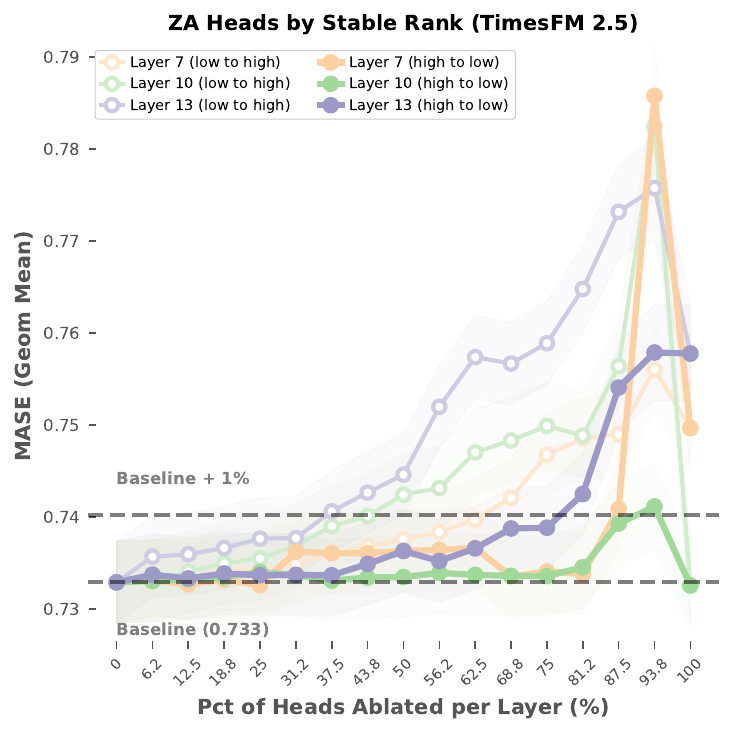}
    \caption{We observe that several of the most ablateable layers fare better when ablating heads from high-to-low stable rank. However, we do not find a consistent general trend with layer depth. As we explained in Section \ref{section:heads_as_kernel_regression} and Appendix \ref{app:proof}, the fundamental limitation to any intrinsic head pruning strategy is that we are not able to consider the data-dependent exponential tilt term of Equation \ref{eq:approxattn}.}
    \label{fig:ablate_srank_high_low_timesfm2p5}
\end{figure}

 Through our head-level ablation experiments (Table \ref{tab:ablations_by_the_numbers_main_text} and Figs. \ref{fig:timesfm2p5_srank_combined}, \ref{fig:srank_ablations_combined_allmodels}, \ref{fig:heads1pp_srank_ablations_combined_allmodels}), we have demonstrated the feasibility of an intrinsic head pruning strategy for TSFMs. This is despite the limitations of any purely intrinsic strategy to account for the data-dependent exponential tilt term of Eq. \ref{eq:approxattn}.

%% file: sections/main_text/section_dataset_view.tex
\section{Consistency Across Datasets} \label{section:dataset_view_main_text}
In Fig. \ref{fig:dataset_view_chronos2_main_text} and Appendix \ref{subsection:dataset_view_appendix} we show that the effect of our ablations are consistent across diverse datasets for the TSFMs in our study, with the exception of Moirai 1.1, for which a few datasets dominate the metrics (Fig. \ref{fig:dataset_view_moirai}).

\begin{figure*}
    \centering
    \includegraphics[width=1.0\linewidth]{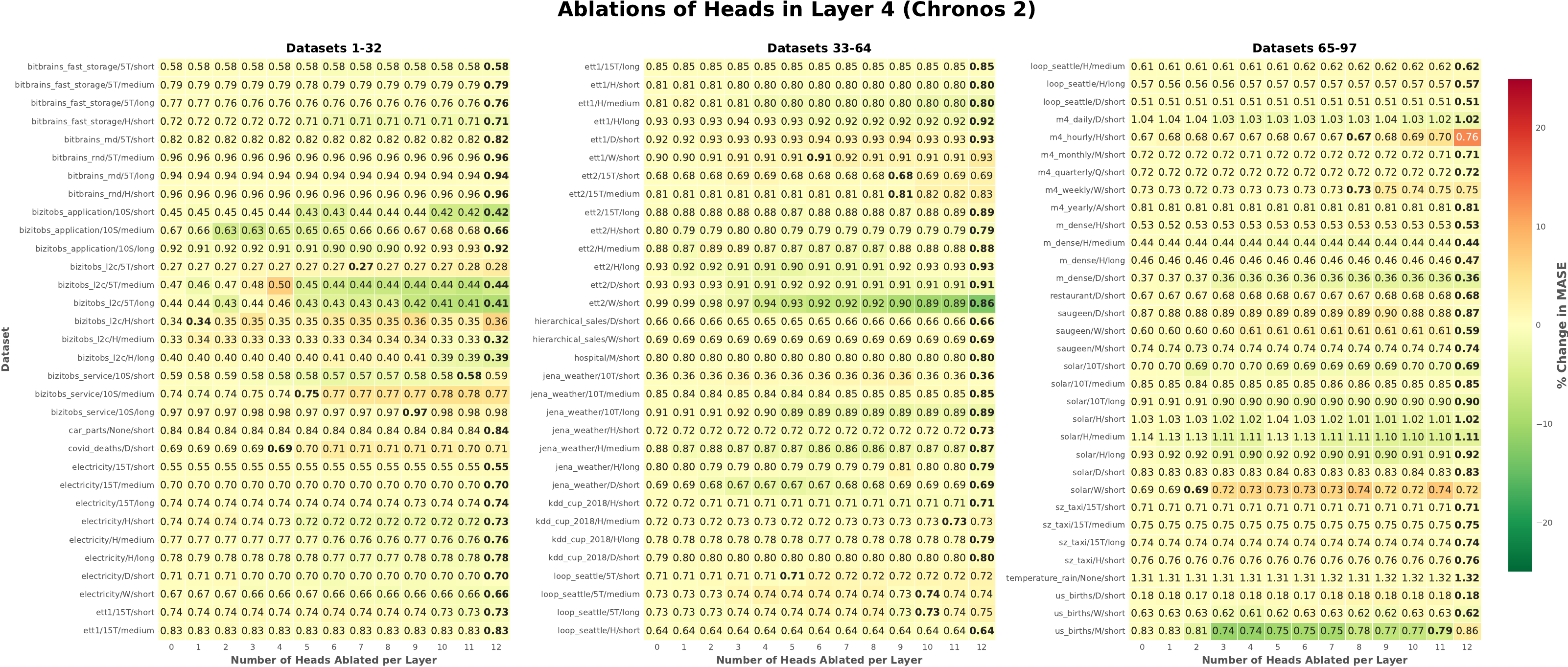}
    \caption{\textbf{Effect of ablations is consistent across diverse datasets}: Our ablation evaluations on GIFT-Eval are not dominated by any single dataset, timescale, or forecast length. To demonstrate this point, we show the performance on 97 datasets in the GIFT-Eval test set under ablations of an increasing number of heads per layer, following our head-level ablation strategy in Section \ref{section:srank_ablations}. We present here the change in MASE (lower i.e. \textcolor{ForestGreen}{decrease} is better, and baseline is \textcolor{GreenYellow}{no change} from original model) on the most ablateable layer of \textit{Chronos 2} (Layer 4). For each dataset, we \textbf{bold} the value at the highest number of heads ablated for which the performance under ablations is within $1\%$ of the original model. In Appendix \ref{subsection:dataset_view_appendix}, Figs \ref{fig:dataset_view_timesfm2p5}, \ref{fig:dataset_view_toto}, \ref{fig:dataset_view_chronos_bolt}, and \ref{fig:dataset_view_moirai} present the per-dataset performance for \textit{TimesFM 2.5}, \textit{Toto}, \textit{Chronos Bolt}, and \textit{Moirai 1.1} respectively, on select layers. We note many datasets on which the ablations actually improve the performance.}
    \label{fig:dataset_view_chronos2_main_text}
\end{figure*}

%% file: sections/main_text/section_related_work.tex
\section{Related Work} \label{section:related_work}

Early TSFMs borrowed architectural cues from LLMs \cite{garza2024timegpt1,ansari2024chronos,jin2024timellm}, however, architectures have evolved as the unique properties of time series as a data class have been identified \cite{yu2025understanding,cohen2023evaluatingrippleeffectsknowledge,garza2024timegpt1,shi2025timemoe}.
As TSFMs enter production, there are relatively few studies on the compression and interpretability of these models, compared to the vast literature on LLMs. Notably, \cite{yu2025understandingtransformerstimeseries} analyze the rank structure through the layer depth of TSFMs, motivating an \textit{extrinsic} (data-dependent) compression method based on applying a truncated SVD to each attention matrix. Our study instead focuses on universal properties across TSFM architectures through ablations of entire components (e.g. attention heads and MLP layers), leading to an \textit{intrinsic} (data-independent) head pruning strategy based on the query-key projection matrix product. 

\textbf{Interpretability}: The widespread adoption of transformers for language modelling has come with attendant efforts for \textit{mechanistic interpretability}: relating the function of a transformer's architectural components to emergent capabilities, such as in-context learning \cite{elhage2021mathematical, olsson2022context}. Several works develop methods to interpret attention mechanisms \cite{serrano-smith-2019-attention, jain-wallace-2019-attention}, early LLMs \cite{Brunner2020On, rogers-etal-2020-primer}, and component subgraphs (i.e. circuits) \cite{vig_causal_mediation, meng2022locating, yao-etal-2023-editing, wang2023interpretability}. Our framework follows \cite{elhage2021mathematical}, including using direct logit attribution (DLA) to illuminate redundant components (Section \ref{subsection:measurements_on_residual_stream}). We also investigate induction heads in \textit{Chronos} (Appendix \ref{section:induction_heads_discussion}). In a similar vein to our work, \cite{wilinski2025exploring} introduce a family of mechanistic interpretability probes that identify correlations among layers \textit{within} models in three TSFM families. The authors find these models are robust to pruning blocks of correlated layers. In contrast, our study finds similarities \textit{across distinct model families}, motivating a more aggressive pruning strategy and a theoretical model. \cite{zhao2026less} develop a structured pruning strategy, progressively removing rows or columns of linear layers in attention or feedforward blocks. They then fine-tune the pruned models. \cite{pandey2025internalsemanticstimeseriesfoundation} study layer specialization in TSFMs, based on canonical motifs like trend and shift, identifying that earlier layers capture local time-domain patterns.


\textbf{Redundancies in Attention Heads}:
Prior works discovered redundant heads in LLMs. \cite{michel2019sixteenheadsreallybetter} observed that models trained with multi-head attention maintain their performance even after pruning a large proportion of their heads---even when only a single head remains in some layers. \cite{mcgrath2023hydraeffectemergentselfrepair} found emergent self-repair: when attention heads in one layer are removed, attention heads in other layers increase their effect on the residual stream to compensate. This form of adaptive computation is also supported by the findings of \cite{voita-etal-2019-analyzing}; only a few select heads are necessary to nearly preserve the model's performance. \cite{hao2021selfattentionattributioninterpretinginformation} implemented informed pruning of heads by attribution score. Our study establishes and explores the same phenomenon for TSFMs, and investigates aspects of model compression unique to the time series domain.

\textbf{MLPs as Key-Value Stores}: Several previous works \cite{geva-etal-2021-transformer, geva-etal-2022-transformer, geva-etal-2023-dissecting, dai-etal-2022-knowledge, yu-ananiadou-2024-neuron, chughtai2024summingfactsadditivemechanisms, merullo2024languagemodelsimplementsimple} have identified the feed-forward, multi-layer perceptron (MLP) layers in transformers as key-value stores. These findings spawned a line of work known as knowledge (fact) editing \cite{meng2022locating, yao-etal-2023-editing, cohen2023evaluatingrippleeffectsknowledge, nanda2023factfinding}, based on the premise that updating the model's MLP parameters can enforce edits to biased, unsafe, and inaccurate content. Despite pitfalls like poor generalization and ineffective usage of edited knowledge \cite{yao-etal-2023-editing, cohen2023evaluatingrippleeffectsknowledge}, knowledge editing represents an instance of interpretability research providing actionable insights for improving models. A parallel line of work has investigated the facts-per-parameter scaling of MLPs \cite{allen-zhu2025physics, zucchet2025how, morris2025languagemodelsmemorize, nichani2024understandingfactualrecalltransformers, dugan2025constructingefficientfactstoringmlps}. Although not our main focus, we find that every TSFM in our study has at least one layer with an MLP that can be ablated without significantly affecting performance (Fig. \ref{fig:single_layer_component_ablations}). This finding and Fig. \ref{fig:components_comparison_bolt_toto} motivates further exploration into the role of MLPs within TSFMs.

\textbf{Role of Depth in Transformers}: Several works have explored the role of depth in LLMs. Recently, \cite{gupta2025llmsusedepth} highlight the importance of depth on task performance and found that LLMs form statistical token guesses in early layers that are refined in later layers, analogous to the spurious lines of thought we observe here in TSFMs. \cite{pmlr-v139-dong21a} show that the MLPs and residual connections in multilayer transformers are crucial for preventing rank collapse in stacked attention layers. Notably, \cite{sanyal2025attentioncollapsesdegeneratelayers} observed collapse of the average rank (over random test set samples) of the attention maps in decoder-only LLMs, despite the simplified settings of prior theory work \cite{pmlr-v139-dong21a,anagnostidis2022signal}. Viewing rank collapse as a drawback of deep models, \cite{sanyal2025attentioncollapsesdegeneratelayers} found that non-collapsed attention layers in large transformers can be used to initialize smaller, compact LLMs which can be trained to match or surpass their larger parents. \cite{saada2025mindgapspectralanalysis} found that softmax attention layers can also exhibit rank collapse with increasing context length. In contrast, we look at the data-independent per-head query-key inner product weight matrix (i.e. $W_Q W_K^\top$) and we empirically demonstrate the importance of the early and late layers across several architectures: encoder-only, encoder-decoder, decoder-only. Finally, recurrent-depth LLMs emulate depth by looping a much smaller block of layers many times, further emphasizing the benefits of avoiding redundant layers or computation in language modeling \cite{geiping2025scalingtesttimecomputelatent}. 

%% file: sections/appendix/appendix_design_space_tsfms.tex
\begin{table}[t]
  \centering
  \caption{Comparison of Time Series Foundation Models (TSFMs) in our study.}
  \label{tab:tsfm_comparison}
  \renewcommand{\arraystretch}{1.0}  
  \setlength{\tabcolsep}{8pt}        
  \resizebox{0.9\textwidth}{!}{%
  \begin{tabular}{p{2.0cm} p{2.5cm} p{3cm} p{9cm}}
    \toprule
    \textbf{Model} & \textbf{Architecture} & \textbf{Tokenization} & \textbf{Key Design Choices} \\
    \midrule

    Chronos & Encoder-Decoder (T5 for cond. generation) & Mean-scale and \newline uniform quantization & Mean-scale the context, then quantize into \(V\) discrete bins (plus \texttt{pad} / \texttt{eos} tokens). Autoregressive decoding (token by token). Trained with cross-entropy (categorical) loss. \\ 
    \midrule
    Chronos-Bolt & Encoder-Decoder (T5) & Patched \newline (non-overlapping) & The decoder directly emits multi-step quantile forecasts. Trained with quantile (pinball) loss.\\ 
    \midrule
    TimesFM & Decoder-only & Patched \newline (non-overlapping) & The decoder predicts the forecast patches causally. Uses autoregressive decoding at patch level. Trained with MSE loss. \\ 
    \midrule
    Toto & Decoder-only & Patched (per-variate causal scaling) & Mixing across channels (variates) using proportional factorized attention. Prediction head is a Student-t mixture, trained with a composite robust / likelihood-based loss. \\
    \midrule
    Moirai & Masked Encoder & Patched (multiple patch lengths) & Multiple projection layers specialized for different frequencies. Any-variate attention. Outputs parameters of mixture distribution. \\
    \midrule
    Chronos 2 & Encoder-only & Patched \newline (non-overlapping) & Interleaved temporal and group (covariate) attention. Based on T5 encoder and trained with a multi-quantile regression loss. \\
    \bottomrule
  \end{tabular}
  }
\end{table}

\section{More Details on TSFMs} \label{section:more_details_tsfms_appendix}

Table \ref{tab:tsfm_comparison} presents a comparison of the key design choices among the TSFMs in our study. \textit{TimesFM 2.5}\footnote{\scriptsize{\url{https://huggingface.co/google/timesfm-2.5-200m-pytorch}}} (200M parameters) has 20 layers and 16 heads per layer. We use the base model sizes for \textit{Chronos}\footnote{\scriptsize{\url{https://huggingface.co/amazon/chronos-t5-base}}} (200M parameters) and \textit{Chronos Bolt}\footnote{\scriptsize{\url{https://huggingface.co/amazon/chronos-bolt-base}}} (205M parameters), which both have 12 layers and 12 heads per layer. \textit{Chronos 2}\footnote{\scriptsize{\url{https://huggingface.co/amazon/chronos-2}}} (120M parameters) has 12 layers and 12 heads per layer.  We use the base model for \textit{Toto 1.0}\footnote{\scriptsize{\url{https://huggingface.co/Datadog/Toto-Open-Base-1.0}}} (151M parameters), which has 12 layers and 12 heads per layer. And we use the base model for \textit{Moirai 1.1}\footnote{\scriptsize{\url{https://huggingface.co/Salesforce/moirai-1.1-R-base}}} (91M parameters), which also has 12 layers and 12 heads per layer. With the exception of \textit{Chronos}, \textit{Chronos Bolt}, and \textit{Chronos 2}, all the aforementioned TSFMs are zero-shot on GIFT-Eval (not pretrained), with no test data leakage. The only publicly available checkpoints we could find for \textit{Chronos} and \textit{Chronos Bolt} were pretrained on GIFT-Eval.

\subsection{Residual Stream}
\label{section:residual_stream_appendix}
The residual stream propagates information across all layers of the model. A simple transformer block with layer norm, self attention, and an MLP module, $\text{LN}(\cdot), \text{SA}(\cdot)$ and $\text{MLP}(\cdot)$, respectively, would have an update rule like the following:
\begin{align*}
    \tilde{H}^{(\ell)} &= \text{LN}(H^{(\ell)}) \\
    H^{(\ell)}_\text{SA} &= H^{(\ell)} + \text{SA}(\tilde{H}^{(\ell)}) \\ 
    \tilde{H}^{(\ell)}_\text{SA} &= \text{LN}(H^{(\ell)}_\text{SA}) \\
    H^{(\ell+1)} &= H^{(\ell)} + \text{MLP}(\tilde{H}^{(\ell)}_\text{SA}),
\end{align*}
where $H$ would be the residual stream. The residual stream is initialized with the embedding of the data, $H^{(0)} = \text{Embed}(X)$, and the output is produced from the final $H^{(L)}$, where $L$ is the number of layers in the model.

\subsection{Design Space}
\label{section:design_space_tsfms_appendix}
\textbf{Architecture:}
While Chronos and Chronos Bolt are encoder-decoder models based on T5, TimesFM 2.5 and Toto are both decoder-only, and Moirai is a masked encoder model.

\textbf{Patching:}
With the exception of Chronos, all these models are patch-based. The patch-based models use non-overlapping patches of contiguous time points. Moirai provides the functionality to use multiple patch sizes via patch-based projections, to adapt larger patch sizes for high-frequency data, and smaller patch sizes for low-frequency data.

\textbf{Probabilistic Forecasts:}
Several models: Chronos, Toto, and Moirai also provide probabilistic forecasts, with a \texttt{num\_samples} option. Meanwhile, Chronos Bolt and TimesFM 2.5 provide 9 quantile forecasts, ranging from the 10th to the 90th percentiles. And Chronos 2 outputs 21 quantiles, uniformly between 0.01 and 0.99.

\textbf{Multivariate Forecasts:}
Toto, Moirai, and Chronos 2 provide the capability for multivariate forecasts. The group attention mechanism of Chronos 2 facilitates in-context learning across multiple univariate timeseries in a group of covariates \cite{faw2025incontext} i.e. "cross-learning" mode. We enable cross-learning for our evaluations, following the default settings.

\subsection{Evaluation Settings}

\begin{table*}[htbp]
  \centering
  \caption{Evaluation Settings for TSFMs in our study.}
  \label{tab:eval_settings_tsfms}
  \renewcommand{\arraystretch}{1.0}  
  \setlength{\tabcolsep}{8pt}        
    \scalebox{0.8}{%
      \begin{tabular}{p{3.0cm} p{2.0cm} p{2.0cm} p{2.0cm} p{2.0cm} p{2.0cm} p{2.0cm}}
        \toprule
        Setting & \textbf{TimesFM 2.5} & \textbf{Chronos 2} & \textbf{Chronos Bolt} & \textbf{Chronos} & \textbf{Toto} & \textbf{Moirai 1.1} \\
        \midrule
        Context Length & 2048 & 8192 & 512 & 512 & 4096 & 4000 \\
        \midrule
        Num Samples & 9 quantiles & 9 quantiles & 9 quantiles & 20 & 20 & 100 \\
        \midrule
        Patch Length & 32 & 16 & 16 & 1 & 32 & 32 \\
        \bottomrule
      \end{tabular}
    }
\end{table*}

Note that several settings deviate from the default settings in the respective codebases, in the interest of reducing the computational time for our extensive ablations experiments. For instance, \textit{TimesFM 2.5} has a maximum context length of 16,384 timepoints; \textit{Chronos 2} can output 21 quantiles; \textit{Toto} defaults to 256 samples; and so on.

%% file: sections/appendix/appendix_zero_ablation.tex
\section{Ablations Methodology} \label{section:ablations_methodology_appendix}

Throughout this paper, we extensively study ablations as our primary tool to understand the importance of individual modules within models. Specifically, we consider zero-ablations which, for some module $f(\cdot)$ that reads from and subsequently adds its output onto the residual stream:
$$\tilde{h} = h + f(h),$$
where $h \in \mathbb{R}^{T \times \text{d\_model} }$ is the residual stream at some position for a sequence of $T$ tokens.

Zero-ablation removes the contribution of the $f$ module from the entire model by fixing $f(h) = 0$, so that the update effectively becomes
$$\tilde{h} = h + f(h) = h.$$

By studying model outputs under zero-ablations for individual or groups of modules, we see how critical they were to the model's computation. For example, if the model's output drastically changes under the ablation of an MLP in the last layer, but less so under the ablation of an MLP in the middle layers of a model, we can infer that the final-layer MLP is more critical to the model's prediction--whereas the middle-layer MLP is more redundant.

%% file: sections/appendix/appendix_more_ablation_experiments.tex
\section{More Ablation Experiments} \label{section:more_ablation_results}
In Section \ref{section:ablations}, we presented key findings from our ablation experiments. Here, we provide more results from our experiments. As before, we conduct large-scale evaluations on the GIFT-Eval test set \cite{aksu2024gifteval}, in addition to the \texttt{dysts} benchmark of chaotic coupled ODE systems \cite{dysts2025} as synthetic data excluded from standard TSFM pretraining.

See Appendix \ref{subsection:additional_results_chronos2} for results on \textit{Chronos 2}, which we defer for the sake of presentation, due to formatting constraints. Also, due to its significantly longer inference time, our ablation experiments on \textit{Chronos} utilize a smaller benchmark of scientific data in the form of low-dimensional chaotic systems of coupled ordinary differential equations (ODEs) \cite{dysts2025}.

\subsection{Ablating Entire Layers}
In Fig. \ref{fig:ablations_spearman}, we presented the Spearman distance between the original model predictions and the predictions under ablations, when ablating entire layers (i.e. all heads and the MLP), or groups of layers. We now present a complementary result in Fig. \ref{fig:ablations_spearman_combined_heads}, in which we \textit{only ablate the heads} of all layers, and leave the MLPs untouched. We still ablate all heads per layer. We note that the interplay between the attention heads and MLPs motivates further investigation via interpretability methods. For instance, when comparing Fig. \ref{fig:ablations_spearman}  with Fig. \ref{fig:ablations_spearman_combined_heads}, we observe that some TSFMs seem to have important MLPs in the later layers that play a crucial role in maintaining the structure of the forecasts. Indeed, Fig. \ref{fig:ablations_spearman_combined_mlps} also supports this observation.

\begin{figure}[htbp]
    \centering
    \includegraphics[width=1.0\linewidth]{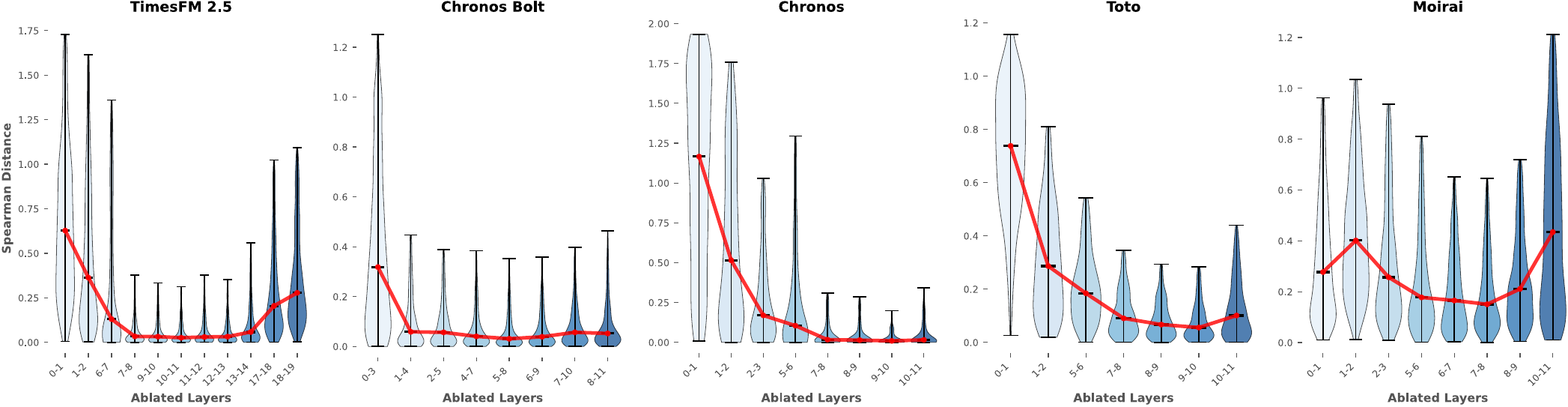}
    \caption{\textbf{Role of Layer Depth in Determining Head Importance}. In contrast to the setting of Fig. \ref{fig:ablations_spearman}, where we ablated the entire layers, we now ablate all heads per layer but keep the MLP untouched. We still observe a depthwise variation in layer importance. This setting is relevant for our ablations of individual heads, presented in Appendix \ref{subsection:more_headlevel_ablations}, in which we also keep the MLPs for each layer.}
    \label{fig:ablations_spearman_combined_heads}
\end{figure} 

\textbf{Spearman Distance}:
We use Spearman Distance $1 - \rho_s(x, y)$ to capture monotonic relationships.
\[
\rho_s(x,y) \;=\;
\begin{cases}
0, & \text{if } \operatorname{Var}(x)=0 \text{ or } \operatorname{Var}(y)=0, \\[6pt]
\operatorname{corr}\,\!\bigl(\operatorname{rank}(x), \operatorname{rank}(y)\bigr),
& \text{otherwise},
\end{cases}
\]
Specifically, we utilize \texttt{scipy.stats.spearmanr}, for the rank-based Spearman correlation:
\[
\operatorname{corr}(a,b)
= \frac{\sum_{t=1}^{T} (a_t-\bar a)(b_t-\bar b)}
{\sqrt{\sum_{t=1}^{T} (a_t-\bar a)^2}
 \sqrt{\sum_{t=1}^{T} (b_t-\bar b)^2}}.
\]
Here, the rank refers to replacing each value with its position in the sorted data (e.g. smallest value gets rank 1, largest value gets rank $T$, and tied values are given the average of the ranks they would occupy). In cases of multivariate forecasts i.e. $x, \, y \in R^{T \times D}$, we return the per-dimension average, over the dimension axis D.

In particular, for the results shown in Figs. \ref{fig:ablations_spearman}, \ref{fig:ablations_spearman_combined_heads}, and \ref{fig:ablations_spearman_combined_mlps}, we compute the Spearman distance for $\approx 2000$ distinct context windows, taken from the GIFT-Eval test set. To ensure that no single dataset dominates this evaluation, we randomly sampled 10 instances from each split i.e. from each combination of dataset name, frequency, and term. Notably, the GIFT-Eval benchmark contains $\approx 1.44 \times 10^{5}$ time series, with up to 21 target variates (dimensions) for some datasets, and the series length for each context window varies considerably across data splits, ranging from 19 to 140,256 (although we set a maximum context length for each model). By conducting our evaluations on GIFT-Eval, we thus ensure that no single frequency/timescale or context length dominates the metrics. The preponderance of short-term predictions with short context windows and short prediction horizons prompted us to roll out the predictions and compute the Spearman distance on the first 64 time points of each prediction. 

\begin{figure}[htbp]
    \centering
    \includegraphics[width=1.0\linewidth]{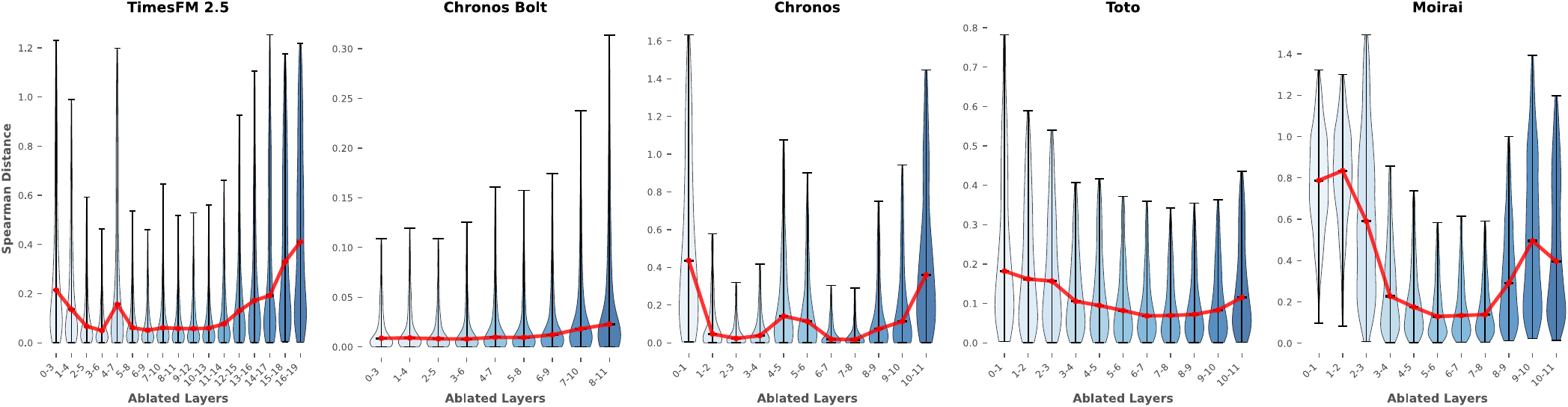}
    \caption{For completeness, we now ablate the MLP for each layer. We still observe a depthwise variation in layer importance. Note again, as with Figs. \ref{fig:ablations_spearman} and \ref{fig:ablations_spearman_combined_heads}, we present the results for some models with multiple consecutive layers under ablation. Also note the scale of the result for Chronos Bolt; the extreme ablateability of the MLPs in Chronos Bolt further agrees with our finding in Fig. \ref{fig:components_comparison_bolt_toto}.}
    \label{fig:ablations_spearman_combined_mlps}
\end{figure} 

\newpage
\subsection{Ablating Individual Heads} \label{subsection:more_headlevel_ablations}
As a supplement to our results in Section \ref{section:srank_ablations}, we present additional findings from our head-level ablations, where we keep the MLPs for each layer but gradually increase the number of heads ablated, following our intrinsic head pruning strategy. 

\begin{figure}[htbp]
    \centering
    \includegraphics[width=1.0\linewidth]{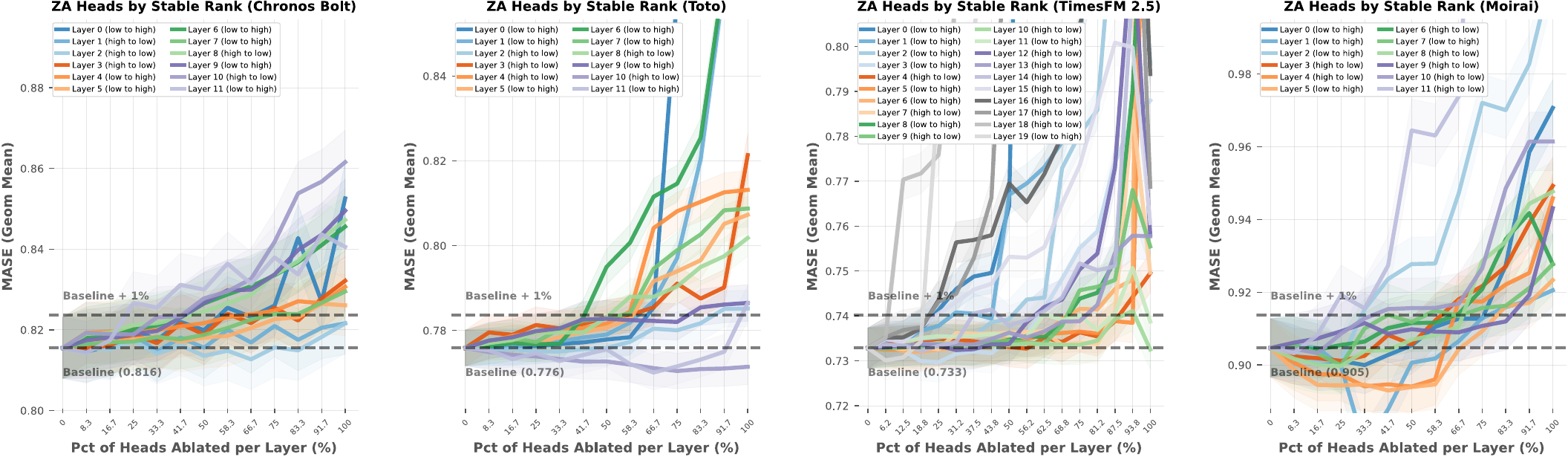}
    \caption{\textbf{Surprising Effectiveness of an Intrinsic Strategy for Ablating Heads}: We zero-ablate heads by order of the stable rank of the query-key projection matrix product. Results are shown for Chronos Bolt, Toto, TimesFM 2.5, and Moirai 1.1, in separate panels from left to right. The legend for each panel marks the ordering i.e. ablate in order of high to low, versus low to high stable rank, that achieves the lowest \texttt{heads@1pp}. We evaluate on the GIFT-Eval test set \cite{aksu2024gifteval}.}
    \label{fig:srank_ablations_combined_allmodels}
\end{figure} 

As elaborated upon in Section \ref{section:srank_ablations}, we observe that in many layers, only a small set of important heads are necessary to maintain the model's performance. In this sense, our findings are consistent with previous findings in the literature \cite{mcgrath2023hydraeffectemergentselfrepair}, which have explored the same phenomenon within LLMs. 

\begin{figure}[H]
    \centering
    \includegraphics[width=1.0\linewidth]{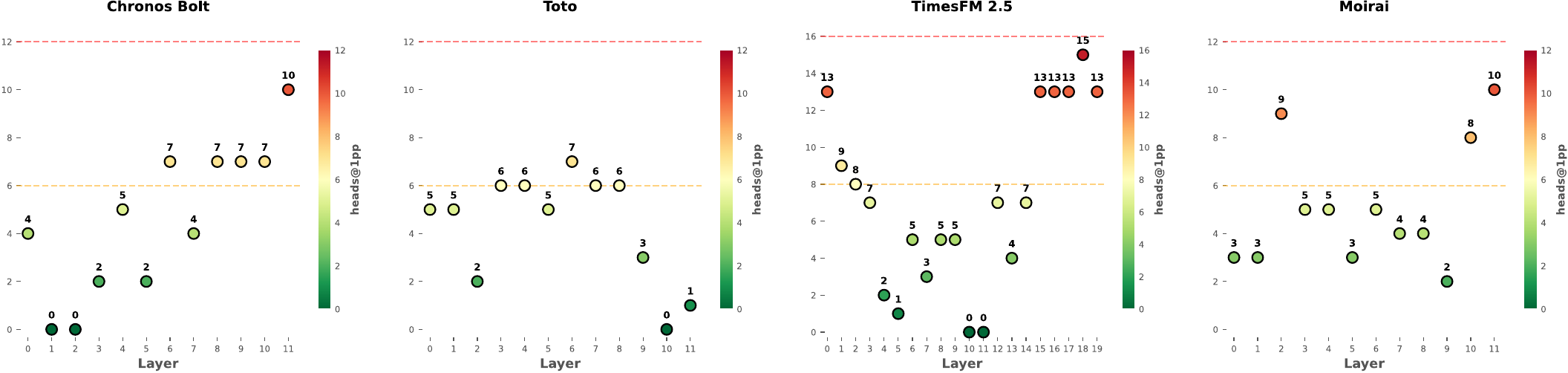}
    \caption{\textbf{Ablateable Layers Require the Few Heads}: To further elucidate the role of layer depth in head ablateability (Fig. \ref{fig:ablations_spearman_combined_heads}), we more succinctly present our findings from Fig. \ref{fig:srank_ablations_combined_allmodels} via our \texttt{heads@1pp} metric (Definition \ref{definition:heads@1pp}). Together with Fig. \ref{fig:srank_ablations_combined_allmodels}, our observations suggest that the layers that are most ablateable (while keeping the MLP) also require the fewest number of heads.}
    \label{fig:heads1pp_srank_ablations_combined_allmodels}
\end{figure} 

\newpage
\subsection{Intrinsic Head Ablation Strategy Preserves the Seasonality and Essential Structure of Forecasts} \label{subsection:preserving_essential_structure}
\textbf{TimesFM 2.5}: In Fig. \ref{fig:timesfm2p5_random_ablate_middle_examples} and Fig. \ref{fig:timesfm2p5_ablate_alllayers_seasonality_versus_random} we demonstrate the advantage of our intrinsic head pruning strategy over ablating randomly-selected heads. All forecasts presented use context length 2048, consistent with our TimesFM 2.5 evaluation setting, but we only show the last 512 timepoints of the context for conciseness.

\begin{figure}[htbp]
    \centering
    \includegraphics[width=0.75\linewidth]{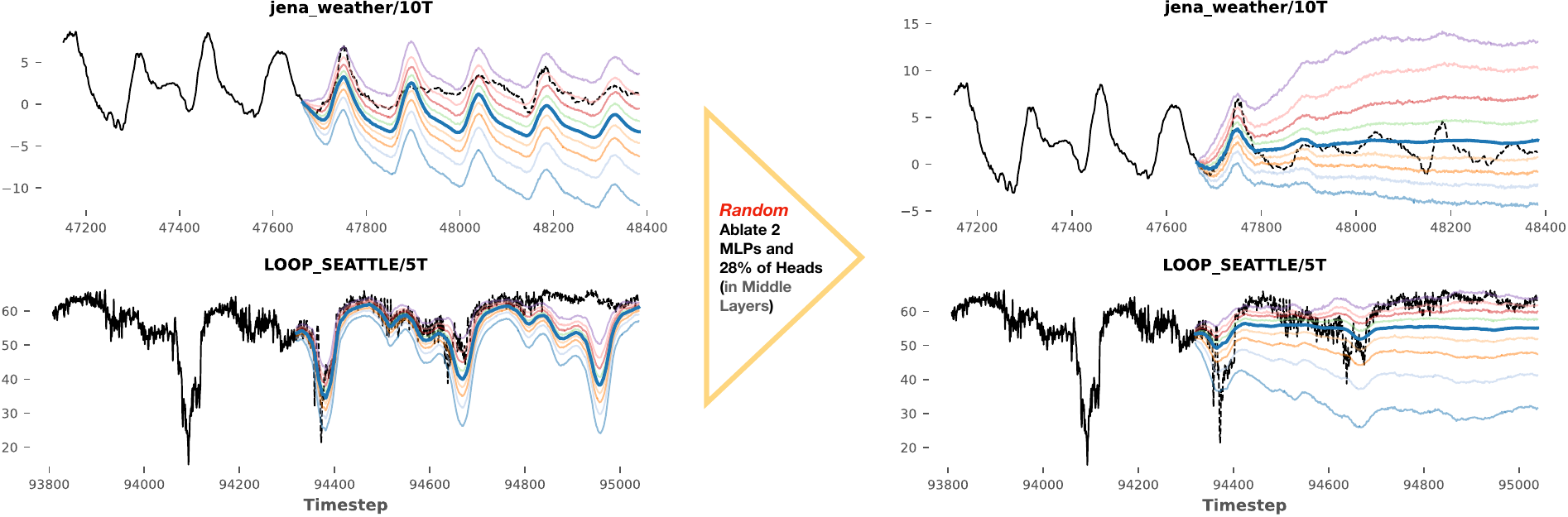}
    \caption{\textbf{Random Head Ablations Destroy the Forecast Structure}: As a companion to Fig. \ref{fig:timesfm2p5_srank_combined}, we show that ablating random heads can have disastrous results on the forecasts. As in Fig. \ref{fig:timesfm2p5_srank_combined}C, we ablate the middle layers of TimesFM 2.5 (layers 7-13 inclusive) down to \texttt{heads@1pp}, but with heads chosen randomly instead of via our intrinsic stable rank strategy. Fig. \ref{fig:timesfm2p5_ablate_alllayers_seasonality_versus_random} presents more examples.}
    \label{fig:timesfm2p5_random_ablate_middle_examples}
\end{figure} 

As we describe in Section \ref{subsection:design_space_failure_modes}, seasonality is an important inductive bias and common behavior exhibited by the leading TSFMs. Fig. \ref{fig:timesfm2p5_ablate_alllayers_seasonality_versus_random} we focus on highly seasonal medium and long-term forecasting tasks from the GIFT-Eval test set.

\begin{figure}[H]
    \centering
    \includegraphics[width=0.9\linewidth]{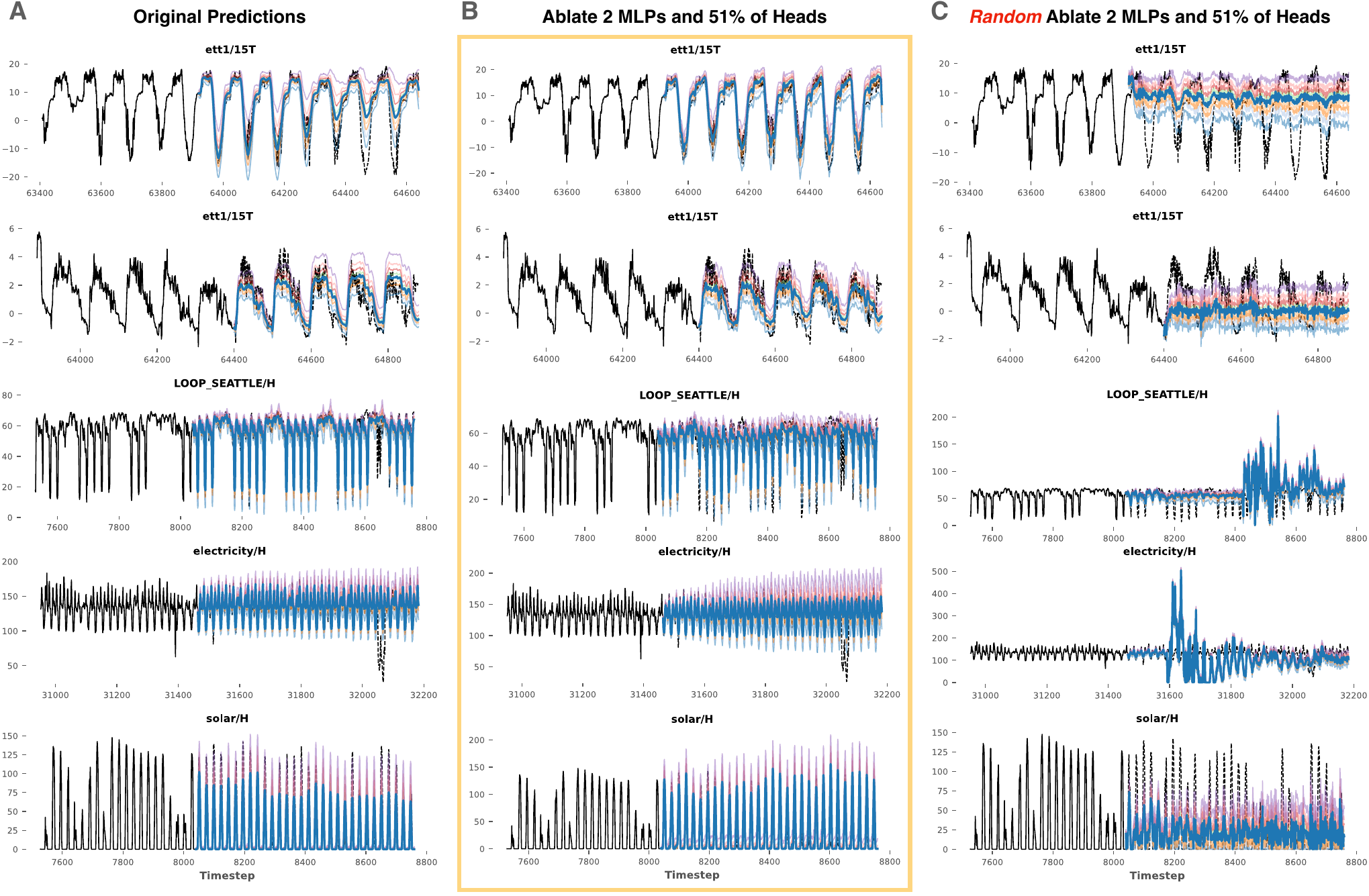}
    \caption{\textbf{Seasonality is Highly Resilient to Informed Head Ablations}: Even when ablating heads from all layers, amounting to $51 \%$ of the total, we observe that highly seasonal forecasts retain their seasonality. We ablate the MLPs of layers 10 and 11, and the heads of the middle layers (layers 7-13 inclusive) down to \texttt{heads@1pp} (as done for Fig. \ref{fig:timesfm2p5_srank_combined} and Fig. \ref{fig:timesfm2p5_random_ablate_middle_examples}); Moreover, we ablate all other layers to their corresponding (1 + \texttt{heads@1pp}). We do this because we observe that ablating the early and late layers to their respective \texttt{heads@1pp} is too aggressive. We present \textbf{(A)} the original forecasts; \textbf{(B)} the forecasts using our intrinsic head ablation strategy described in Section \ref{section:srank_ablations}; and \textbf{(C)} the forecasts after ablating the same MLPs, but randomly-selected heads, illustrating the advantage of our approach.}
    \label{fig:timesfm2p5_ablate_alllayers_seasonality_versus_random}
\end{figure} 

\newpage
\textbf{Toto}:
\begin{figure}[htbp]
    \centering
    \includegraphics[width=0.75\linewidth]{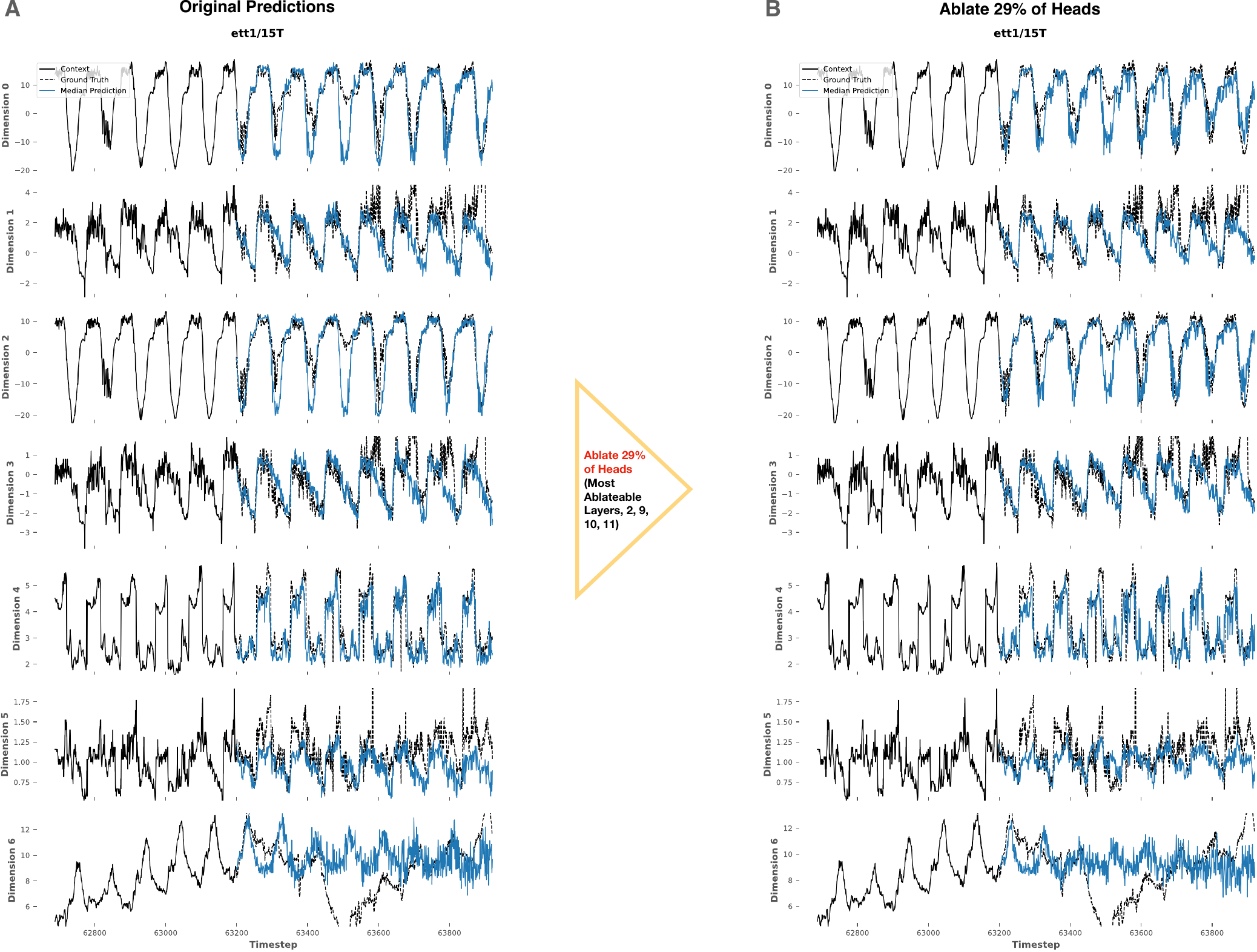}
    \caption{For an example \textit{multivariate} forecasting task using \textit{Toto} \textbf{(A)}, we ablate $29\%$ of all heads by taking the most ablateable layers to their corresponding \texttt{heads@1pp} \textbf{(B)}. We identified these layers (Layers 2, 9, 10, 11) through our ablations of entire layers (Fig. \ref{fig:ablations_spearman_combined_heads}).}
    \label{fig:toto_ablations_example_most_ablateable}
    \vskip -0.2in
\end{figure} 

\textbf{Chronos Bolt}:
\begin{figure}[H]
    \centering
    \includegraphics[width=0.75\linewidth]{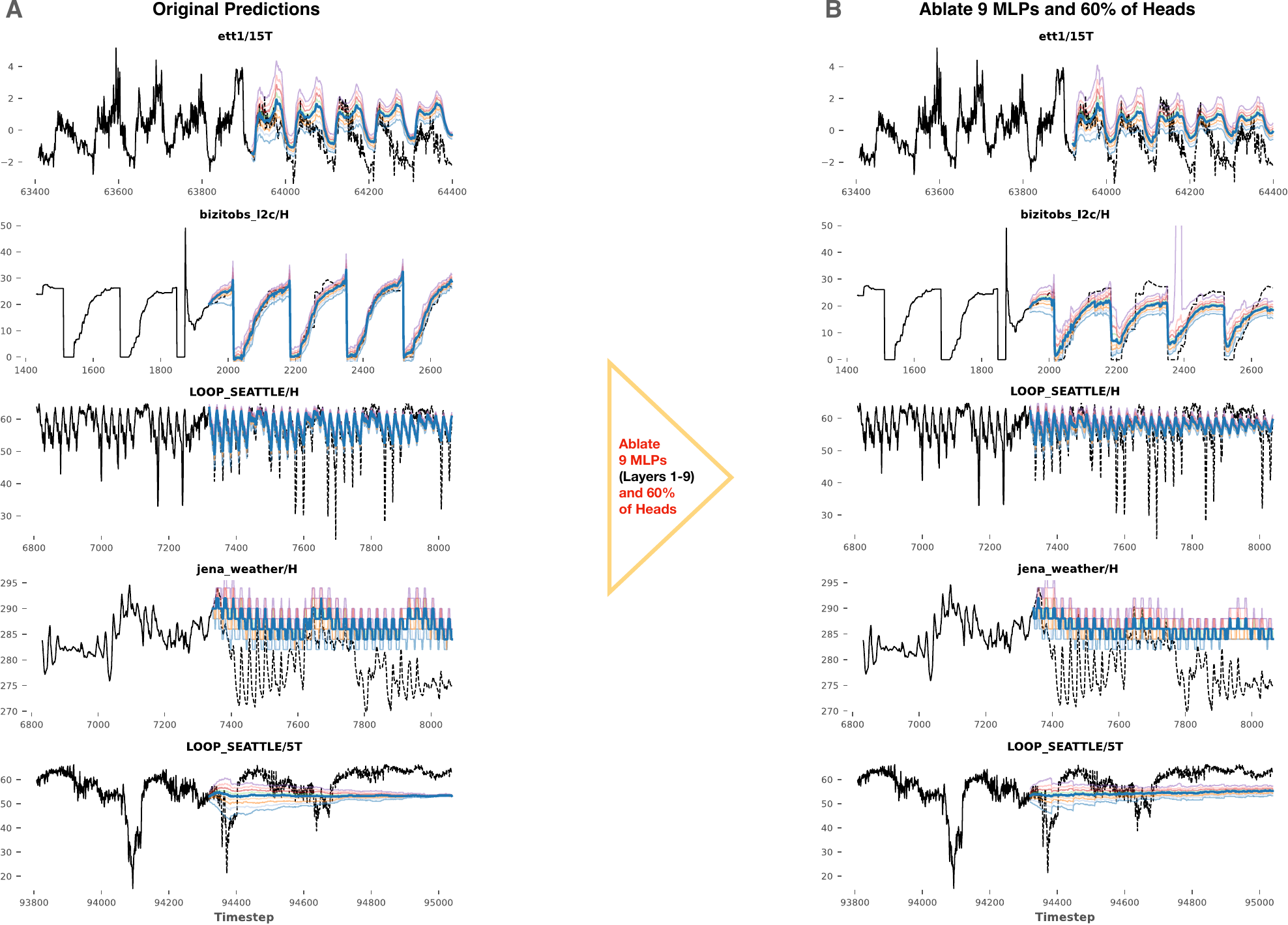}
    \caption{For several forecasting tasks using \textit{Chronos Bolt} \textbf{(A)}, we ablate 9 MLPs (in Layers 1-9 inclusive) and $60\%$ of all heads by taking Layers 0-10 (inclusive) to their \texttt{heads@1pp} \textbf{(B)}. Our experiments suggest Chronos Bolt is particularly resilient to ablations.}
    \label{fig:ablations_examples_chronos_bolt}
    \vskip -0.2in
\end{figure} 

\textbf{Chronos 2}:
\begin{figure}[H]
    \centering
    \includegraphics[width=1.0\linewidth]{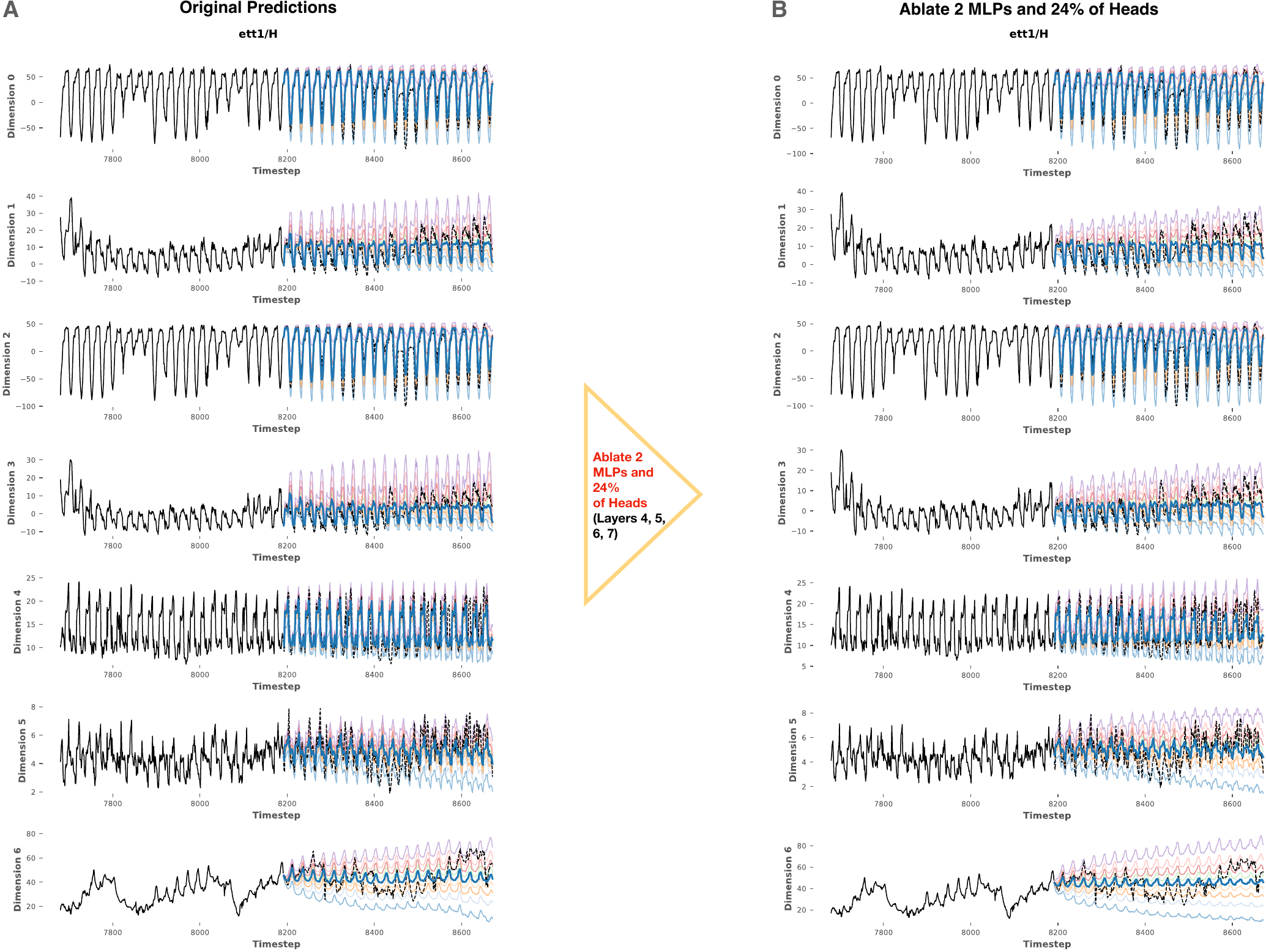}
    \caption{For an example multivariate forecasting task using \textit{Chronos 2} \textbf{(A)}, we ablate 2 MLPs (in Layers 4 and 5) and $24\%$ of all heads (in Layers 4, 5, 6, 7).}
    \label{fig:ablations_examples_chronos2}
    \vskip -0.2in
\end{figure} 

\newpage
\subsection{Performance on Individual Datasets} \label{subsection:dataset_view_appendix}
Our ablation evaluations on GIFT-Eval are not dominated by any single dataset, timescale, or forecast length. To demonstrate this point, we examine the performance of various TSFMs on each of the 97 datasets in the GIFT-Eval test set under ablations of an increasing number of heads per layer, following our head-level ablation strategy in Section \ref{section:srank_ablations}. Figs. \ref{fig:dataset_view_timesfm2p5}, \ref{fig:dataset_view_toto}, \ref{fig:dataset_view_chronos_bolt}, \ref{fig:dataset_view_chronos2}, and \ref{fig:dataset_view_moirai} present the per-dataset performance for \textit{TimesFM 2.5}, \textit{Toto}, \textit{Chronos Bolt}, \textit{Chronos 2}, and \textit{Moirai 1.1} respectively.

For each dataset, we \textbf{bold} the value at the highest number of heads ablated for which the performance under ablations is within $1\%$ of the original model. For each model, we show this dataset-level view for an important layer (i.e. first or last layer) versus for the most ablateable layer of that model.

Specifically, we present the per-dataset performance as heatmaps colored by the percent change in MASE on each dataset, at each ablation level (number of heads ablated) for the specified layer. For the sake of presentation, we cap the maximum decrease (performance improvement) and increase (performance degradation) at $-25\%$ and $25\%$ change in MASE, respectively. In other words, \textcolor{PineGreen}{Green} denotes improvement, whereas \textcolor{BrickRed}{Red} denotes degradation in the ablated model's performance. The numerical values overlayed on each cell are the actual values of the MASE metric, rounded to two decimal places.

As we can see with particular clarity in Fig. \ref{fig:dataset_view_timesfm2p5}, even the ablation of all heads in Layer 10 of \textit{TimesFM 2.5} preserves the MASE for nearly every dataset, whereas the ablation of Layer 0 introduces catastrophic performance degradation ($\ge 25\%$ increase in MASE) across almost all the datasets. \textit{TimesFM 2.5} is both the deepest and widest model evaluate (as well as the best-performing), with 20 layers and 16 heads per layer.

We note that under ablations of heads, the performance of \textit{Moirai 1.1} on GIFT-Eval appears to be dominated by a few datasets, which significantly impact the metrics (Fig. \ref{fig:dataset_view_moirai}).

\begin{figure}[htbp]
    \centering
    \includegraphics[width=0.9\linewidth]{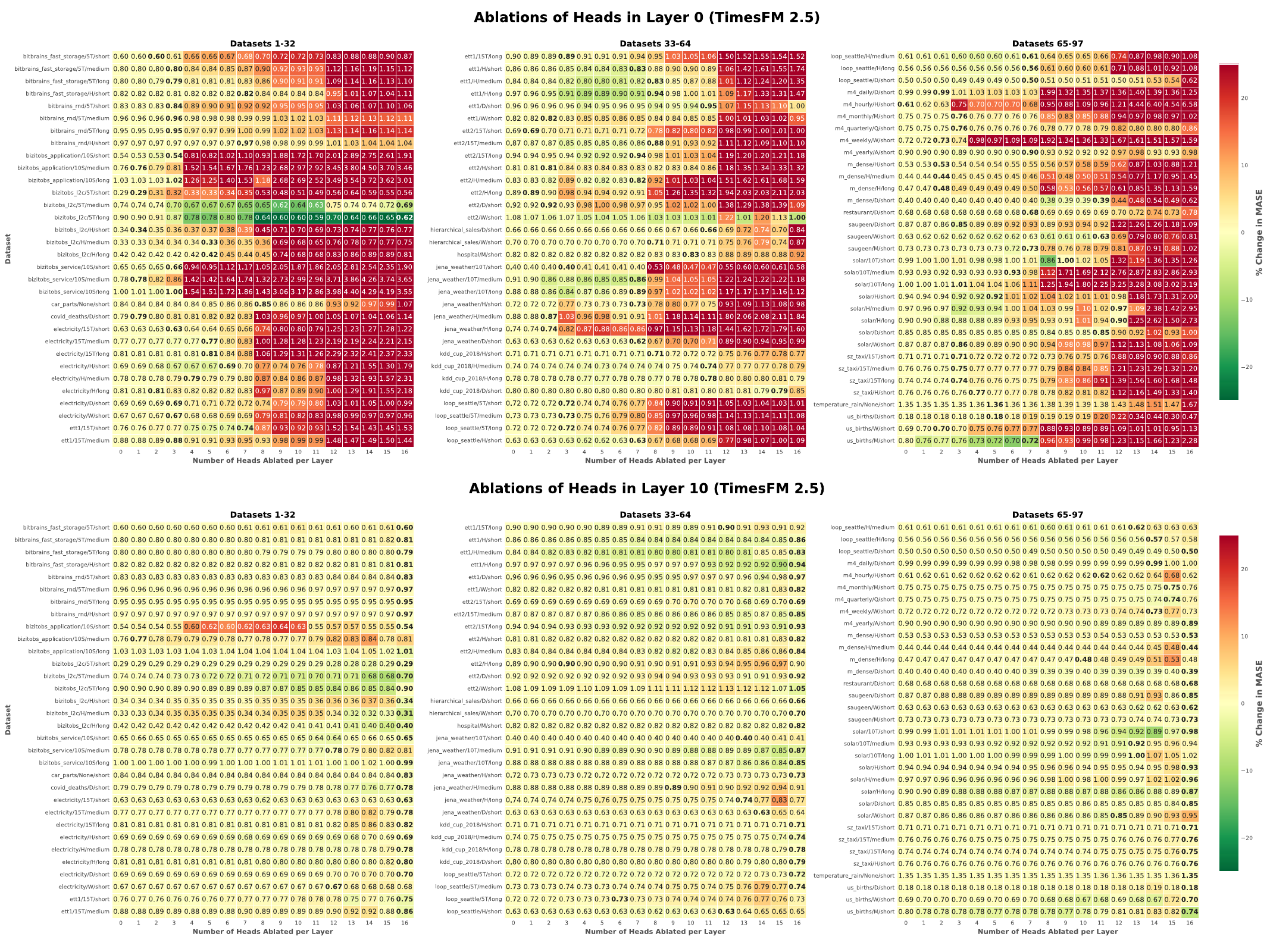}
    \caption{Per-dataset performance of \textit{TimesFM 2.5} under increasing head ablations on Layer 0 \textbf{(Top)} versus Layer 10 \textbf{(Bottom)}.}
    \label{fig:dataset_view_timesfm2p5}.
\end{figure} 

\begin{figure}[htbp]
    \centering
    \includegraphics[width=0.8\linewidth]{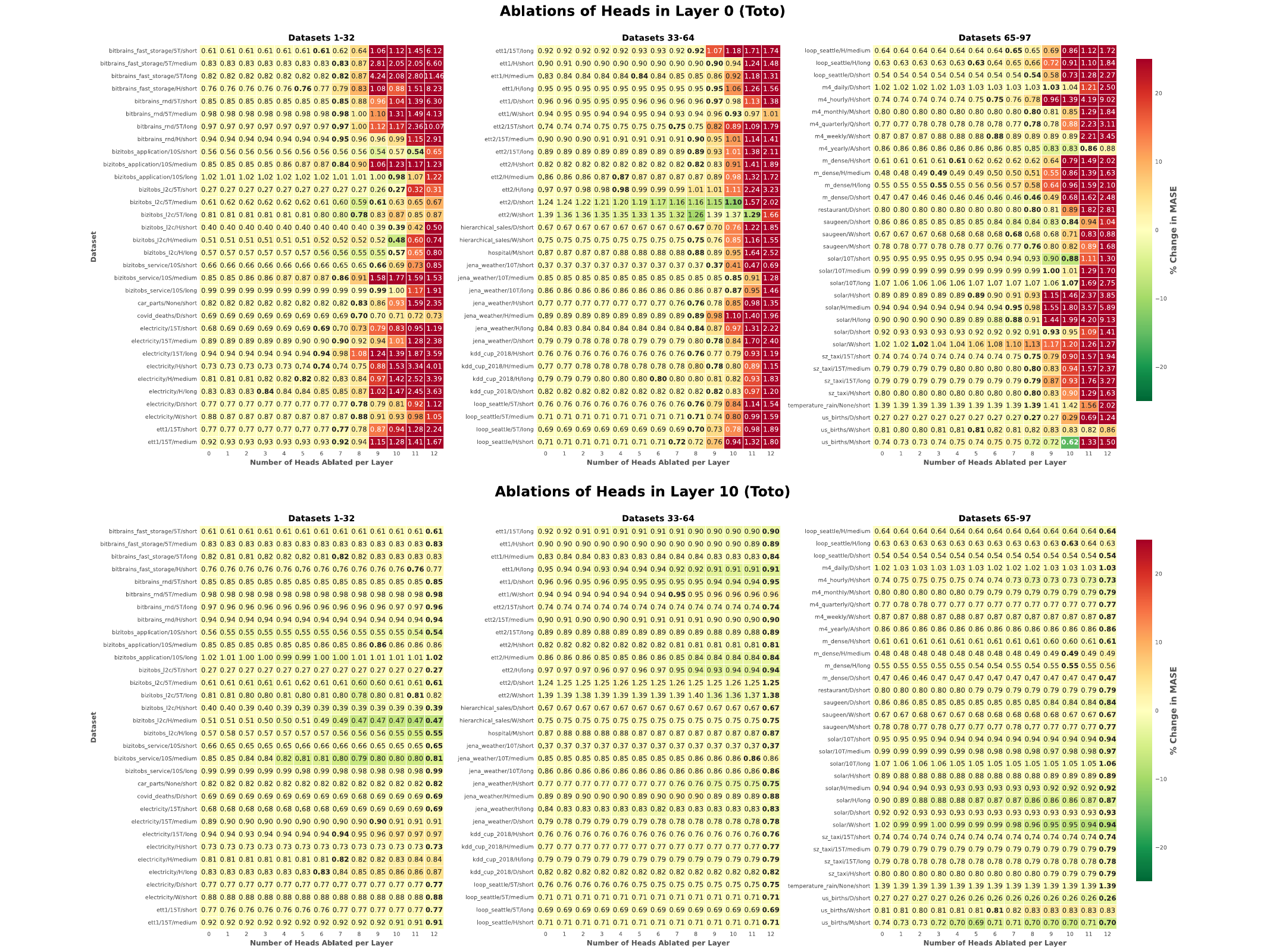}
    \caption{Per-dataset performance of \textit{Toto} under increasing head ablations on Layer 0 \textbf{(Top)} versus Layer 10 \textbf{(Bottom)}.}
    \label{fig:dataset_view_toto}.
\end{figure} 

\begin{figure}[htbp]
    \centering
    \includegraphics[width=0.8\linewidth]{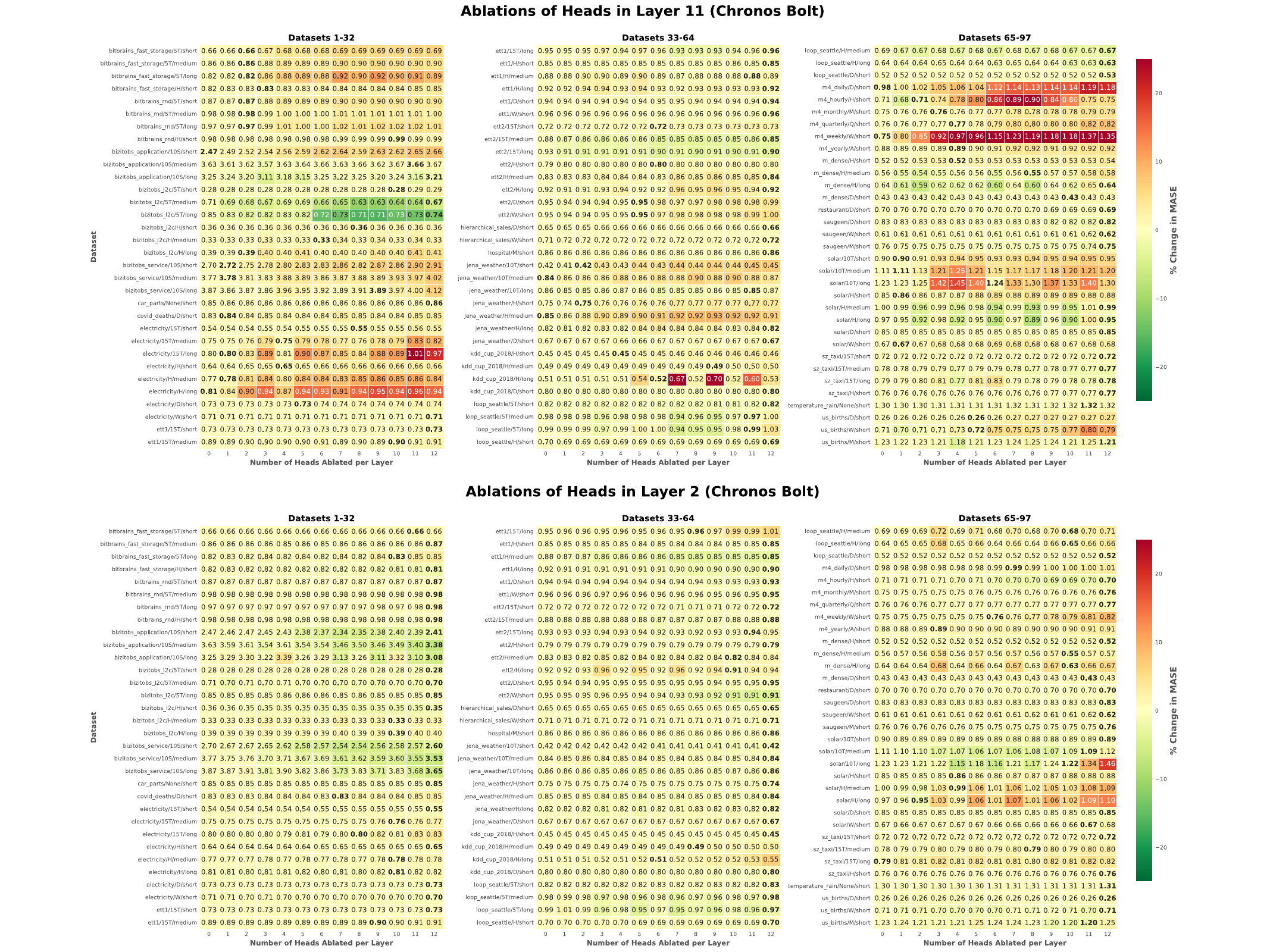}
    \caption{Per-dataset performance of \textit{Chronos Bolt} under increasing head ablations on Layer 11 \textbf{(Top)} versus Layer 2 \textbf{(Bottom)}.}
    \label{fig:dataset_view_chronos_bolt}.
\end{figure} 

\begin{figure}[htbp]
    \centering
    \includegraphics[width=0.8\linewidth]{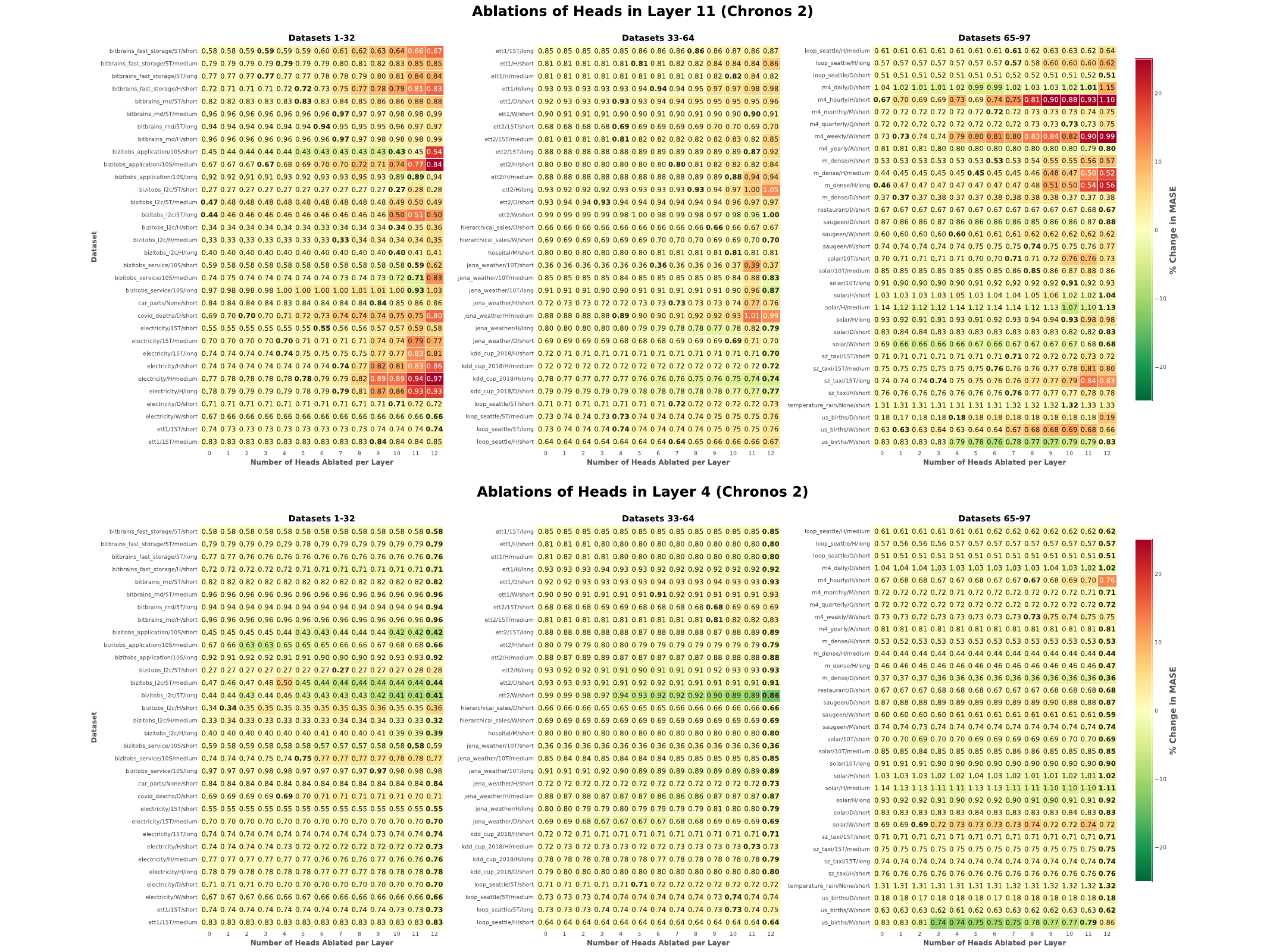}
    \caption{Per-dataset performance of \textit{Chronos 2} under increasing head ablations on Layer 11 \textbf{(Top)} versus Layer 4 \textbf{(Bottom)}.}
    \label{fig:dataset_view_chronos2}.
\end{figure} 

\begin{figure}[htbp]
    \centering
    \includegraphics[width=0.8\linewidth]{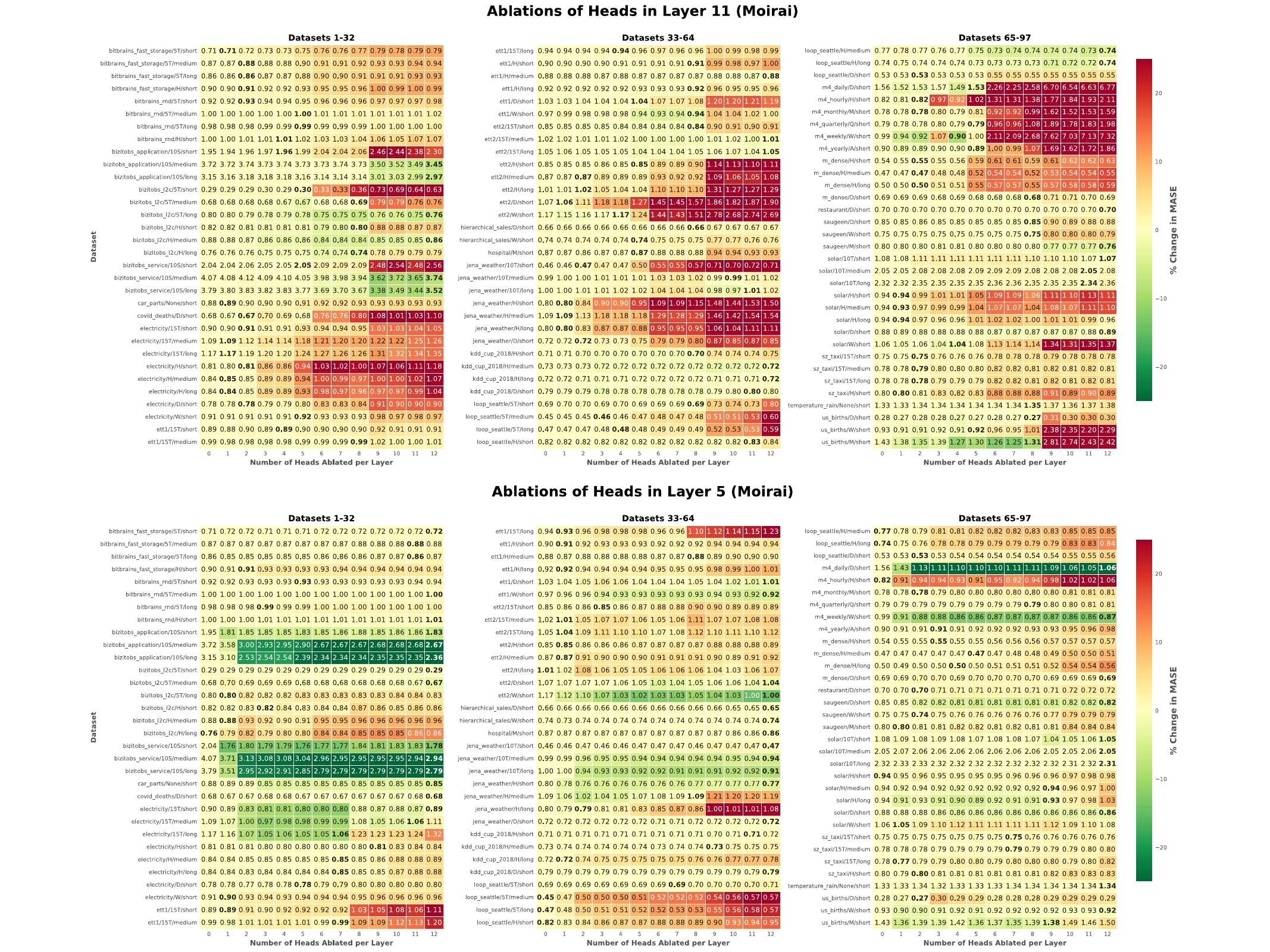}
    \caption{Per-dataset performance of \textit{Moirai 1.1} under increasing head ablations on Layer 11 \textbf{(Top)} versus Layer 5 \textbf{(Bottom)}.}
    \label{fig:dataset_view_moirai}.
\end{figure} 

\begin{table*}[htbp]
\centering
\scalebox{0.86}{%
\begin{tabular}{lccccccc}
\multicolumn{8}{c}{\textbf{Evaluation of Ablated Models on GIFT-Eval test set}} \\
\toprule
Model \footnotesize{(ZA $\textcolor{WildStrawberry}{\% \text{Heads}} + \textcolor{Periwinkle}{N_{\text{MLP}}}$)}
& \footnotesize{$\% \Delta$ MASE $(\downarrow)$}
& \footnotesize{$\% \Delta$ NRMSE $(\downarrow)$}
& \footnotesize{$\% \Delta$ sMAPE $(\downarrow)$}
& \footnotesize{$\% \Delta$ CRPS $(\downarrow)$}
& \footnotesize{$\% \Delta$ MSIS $(\downarrow)$}
& \footnotesize{$N_{\text{Resilient}}$}
& \footnotesize{$N_{\text{Improved}}$} \\
\midrule
\textbf{Chronos 2 \footnotesize{(\textcolor{WildStrawberry}{24\%} + \textcolor{Periwinkle}{2 MLP})}} &
\textbf{4.71\%} & 2.29\% & 3.44\% & \textbf{3.92\%} & 9.47\% &
$36$ & $24$ \\
\midrule
\textbf{TimesFM 2.5 \footnotesize{(\textcolor{WildStrawberry}{28\%} + \textcolor{Periwinkle}{2 MLP})}} &
\textbf{6.12\%} & 2.46\% & 4.72\% & \textbf{5.86\%} & 6.59\% &
$33$ & $12$ \\
\midrule
\textbf{Toto \footnotesize{(\textcolor{WildStrawberry}{29\%})}} &
\textbf{3.54\%} & -38.00\% & 2.73\% & \textbf{3.60\%} & -24.64\% &
$41$ & $30$ \\
Toto \footnotesize{(\textcolor{WildStrawberry}{29\%} + \textcolor{Periwinkle}{1 MLP})} &
8.54\% & -28.24\% & 5.77\% & 8.73\% & -15.27\% &
$33$ & $13$ \\
Toto \footnotesize{(\textcolor{WildStrawberry}{43\%})} &
10.10\% & -8.19\% & 6.36\% & 9.97\% & 8.24\% &
$26$ & $14$ \\
\midrule
\textbf{Chronos Bolt \footnotesize{(\textcolor{WildStrawberry}{28\%} + \textcolor{Periwinkle}{6 MLP})}} &
\textbf{4.24\%} & -0.77\% & 2.46\% & \textbf{0.19\%} & -15.89\% &
$32$ & $27$ \\
Chronos Bolt \footnotesize{(\textcolor{WildStrawberry}{42\%} + \textcolor{Periwinkle}{6 MLP})} &
10.71\% & 1.81\% & 6.18\% & 6.42\% & -9.75\% &
$29$ & $11$ \\
\bottomrule
\end{tabular}
}
\vspace{6pt}
\caption{We evaluate leading TSFMs under heavy ablations of heads found via our intrinsic head-pruning strategy (Section \ref{section:srank_ablations}), in addition to some MLPs. We report the percent change (lower is better) in MASE geometric mean across the 97 splits (i.e. specification of dataset, frequency, term) of the GIFT-Eval test set. Also shown: the number of resilient datasets ($N_{\text{Resilient}}$) on which the model's performance degraded (i.e. MASE increased) by less than $5\%$ over the original model, and the number of improved datasets ($N_{\text{Improved}}$) whose performance actually improved (in MASE) under the ablations. We highlight \textcolor{WildStrawberry}{percent of total heads ablated} and \textcolor{Periwinkle}{number of MLPs ablated}. We \textbf{bold} the ablations for which we simply take $\approx 1/3$ of all layers (identified to be most ablateable in Appendix \ref{section:more_ablation_results}) to \texttt{heads@1pp}. For these ablations, we specifically bold the MASE and CRPS metrics because they are reported by GIFT-Eval.}
\label{tab:ablations_by_the_numbers}
\end{table*}

\subsection{Additional Results for Chronos 2} \label{subsection:additional_results_chronos2}

\begin{table*}[htbp]
\centering
\caption{Percent degradation of \textit{Chronos-2} on GIFT-Eval, under ablations of entire layers (Heads and MLP). We present the MASE and CRPS, as reported by the GIFT-Eval leaderboard, in addition to several other metrics.}
\label{tab:layer_metric_degradation_chronos2}
\footnotesize
\begin{tabular}{lccccc}
\toprule
\textbf{Layer} & \textbf{MASE} & \textbf{sMAPE} & \textbf{MSIS} & \textbf{CRPS} \\
\midrule
Layer 0  & 6.39\% & 5.10\% & 10.82\% & 4.68\% \\
Layer 1  & 1.02\% & 2.36\% & 6.92\%  & 0.53\% \\
Layer 2  & 0.31\% & 0.14\% & 2.54\% & 0.15\% \\
Layer 3  & 0.85\% & 0.91\% & 1.71\% & 0.66\% \\
\textbf{Layer 4}  & \textbf{0.00\%} & 0.62\% & 0.91\% & \textbf{-0.34\%} \\
Layer 5  & 0.16\% & 0.55\% & -0.18\% & -0.02\% \\
Layer 6  & 1.89\% & 1.14\% & -0.07\% & 1.37\% \\
Layer 7  & 0.96\% & 1.62\% & 2.29\% & 0.61\% \\
Layer 8  & 1.61\% & 2.01\% & 4.49\% & 1.14\% \\
Layer 9  & 1.05\% & 1.55\% & 3.78\% & 0.74\% \\
Layer 10 & 3.26\% & 3.53\% & 2.68\% & 2.44\% \\
Layer 11 & 20.10\% & 14.59\% & 27.99\% & 19.06\% \\
\bottomrule
\end{tabular}
\end{table*}

\begin{figure}[H]
    \centering
    \includegraphics[width=1.0\linewidth]{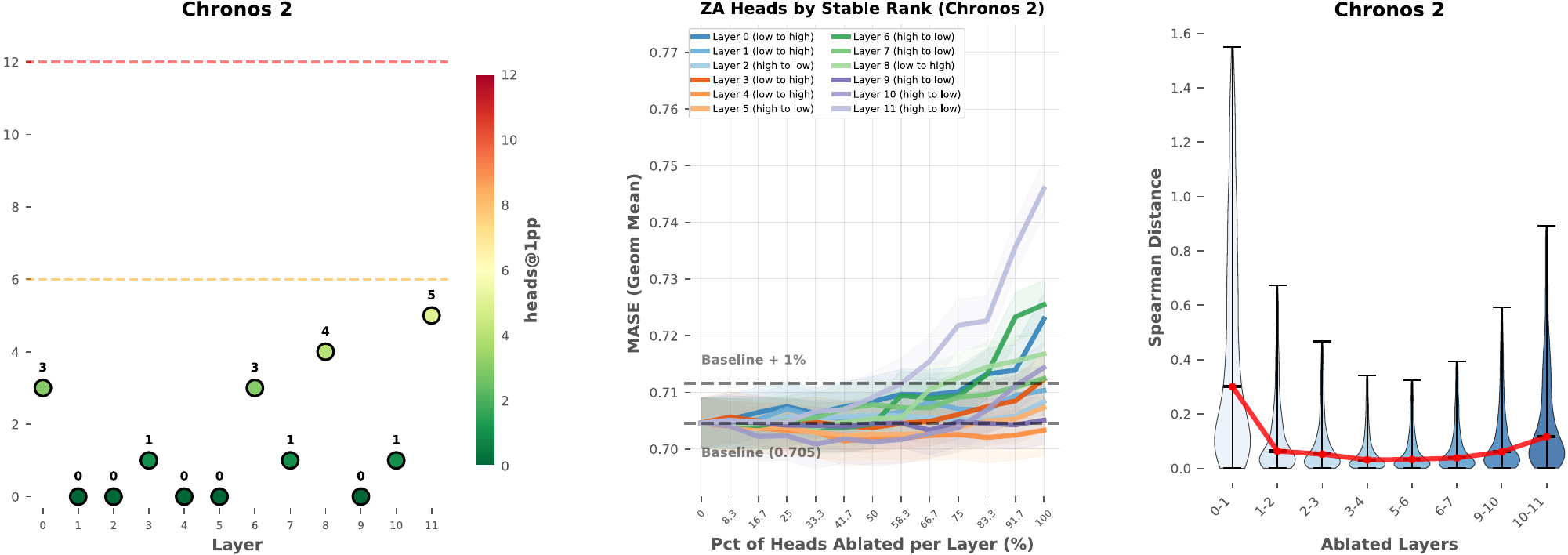}
    \caption{For \textit{Chronos 2}, we present: \textbf{Left}: \texttt{heads@1pp} per layer; \textbf{Middle}: MASE after zero-ablating heads by order of stable rank of the query-key projection matrix product; and \textbf{Right}: the Spearman distance between the original predictions and the predictions under ablations of entire layers.}
    \label{fig:extra_chronos2_results}.
\end{figure} 

%% file: sections/appendix/appendix_spectral_sharpness.tex
\section{Proofs}
\label{app:proof}

We now formally restate and prove Proposition \ref{prop:concentration} which establishes that the singular value weighted query norm in the top-r subspace (and tail error) drives sharpness in the attention matrix of cross-attention heads.
\begin{proposition}
\label{prop:concentration_formal}
Let $\bM = \bW_q \bW_k^\top/\sqrt{d_\text{head}}$ be a cross attention head with SVD $\bM = \bU \bSigma \bV^\top$. For $0 < r < d_\text{head}$, let $\bM_r \coloneqq \bU_r\bSigma_r\bV^\top_r$ be the rank-$r$ truncation of the SVD and assume the spectral tail satisfies $\sum_{i>r}\sigma_i^2\leq \rho$ where $\sigma_i = \bSigma_{ii}$. 

Let $\{\bk_j\}_{j=1}^C$ be the encoder context keys, $\{\bq_i\}_{i=1}^H$ the decoder queries, and define $M_q \coloneqq \max_{i \in [H]} ||\bq_i||_2$ and $M_k \coloneqq \max_{j \in [C]} ||
\bk_j||_2$. Define the top-r key and query features as $\tilde\bk_j^r \coloneqq \bV_r^\top \bk_j$ and $\tilde{\bq}_i^r \coloneqq \bU^\top_r \bq_i$, respectively.

Moreover, define $\ell_{ij} \coloneqq \langle\bq_i, \bM\bk_j\rangle$ and $\ell_{ij}^r \coloneqq \langle\bSigma_r\tilde\bq_i^r, \tilde\bk_j^r\rangle$ as the score and truncated score, respectively. Assume for any $i \in [H]$ there exists a $j^*(i) = \arg\max_{j\in[C]}\ell_{ij}^r$, margin $\gamma_i^r > 2\sqrt\rho M_q M_k > 0$, and constants $a_i,d_i > 0$ such that,
\begin{enumerate}[label=(\roman*)]
  \item (Score margin) $\ell_{ij^*(i)}^r\geq \ell_{ij}^r + \gamma_i^r \quad \forall j \ne j^*(i)$
  \item (Key separation) $||\tilde{\bk}_{j^*(i)}^r - \tilde{\bk}_{j^\dagger}^r||_2 > d_i$
  \item (Alignment separation) If $\theta_i$ is the angle between $\bSigma_r\tilde\bq_i^r$ and $\tilde{\bk}_{j^*(i)}^r - \tilde{\bk}_{j^\dagger}^r$, then $\cos\theta_i > a_i$
\end{enumerate}
Where $j^\dagger \in \arg\max_{j\ne j^*(i)} \ell_{ij}^r $ is the key index of the second highest score. Then the SM attention matrix $\bA_{ij}$ with inverse temperature $\beta > 0$, concentrates around the selector matrix $\delta_{i,j^*(i)}$ with rate:
\[
1-\bA_{i, j^*(i)} \leq (C-1)\exp(-\beta (a_i d_i||\bSigma_r\tilde{\bq}^r_i||_2 - 2\sqrt\rho M_q M_k))
\]
Where $a_i > 0$ and $d_i > 0$.
\end{proposition}

\begin{proof}
    The softmax on row $i$ concentrates on column $j^*(i)$ with the following general bound:
    \begin{align*}
    1-\bA_{i,j^*(i)} &= \sum_{j \ne j^*(i)}\exp(\beta\ell_{ij}) / \sum_k \exp(\beta\ell_{ik}) \\
    &= \sum_{j \ne j^*(i)}\exp(\beta(\ell_{ij} - \ell_{ij^*(i)})) / (1+\sum_{k\ne j^*(i)} \exp(\beta(\ell_{ik}-\ell_{ij^*(i)}))) \\
    &\leq \sum_{j \ne j^*(i)}\exp(\beta(\ell_{ij} - \ell_{ij^*(i)})) \\ 
    &\leq (C-1)\exp(-\beta(\ell_{ij^*(i)}-\max_{j\ne j^*(i)}\ell_{ij}))
    \end{align*}
    Let $\gamma_i \coloneqq \ell_{ij^*(i)}-\max_{j\ne j^*(i)}\ell_{ij}$ be the non-truncated top-2 score gap and $\gamma_i^r \coloneqq \ell_{ij^*(i)}^r-\ell_{ij^\dagger}^r$ which is equivalent to the margin condition. To control the error between the score and truncated score, let $\bE \coloneqq \bM - \bM_r$ and note that $||\bE||_2 \leq ||\bE||_F \leq \sqrt{\rho}$. This implies that $|\ell_{ij} - \ell_{ij}^r | \leq M_qM_k\sqrt{\rho}$ and consequently, $\gamma_i \geq \gamma_i^r - 2\sqrt\rho M_qM_k > 0$. Thus,
    \begin{align*}
        1-\bA_{i,j^*(i)} &\leq (C-1)\exp(-\beta\gamma_i) \leq (C-1)\exp(-\beta(\gamma_i^r - 2\sqrt\rho M_qM_k))
    \end{align*}
    And since,
    \[
        \gamma_i^r = \ell_{ij^*(i)}^r - \ell_{ij^\dagger}^r 
        = \langle \bSigma_r \tilde{\bq}_i^r, \tilde{\bk}_{j^*(i)}^r - \tilde{\bk}_{j^\dagger}^r\rangle 
        = ||\bSigma_r \tilde{\bq}_i^r||_2 ||\tilde{\bk}_{j^*(i)}^r - \tilde{\bk}_{j^\dagger}^r||_2 \cos\theta
        \geq ||\bSigma_r \tilde{\bq}_i^r||_2 d_i a_i
    \]
    The desired bound follows.
    
\end{proof}

As discussed in Section \ref{subsection:spectral}, if parroting occurs and induces a motif in the decoder queries, produced by autoregressive rollout, which aligns with a motif in the context, then the cross attention matrix unsurprisingly exhibits a local shift structure. The above result relates the truncated spectrum of the cross attention head to the \textit{sharpness} of such a shift matrix. 

We note that the above result is technically limited to \textit{Chronos}-style encoder-decoder models that employ causal masking. Nonetheless, we find that ablating attention heads according to stable rank is a generally effective compression strategy even for decoder-only and encoder-only TSFMs.

%% file: sections/appendix/appendix_extended_kernel_discussion.tex
\section{Extended Discussion on Sharp Heads} \label{section:extended_discussion_attention_kernel_regression}

\subsection{Empirical Observations in \textit{Chronos}}
In Fig. \ref{fig:head_sharpness} we present our empirical observation of sharp and diffusion heads present in the \textit{Chronos} decoder. 


While Fig. \ref{fig:head_sharpness} demonstrated the preservation of context parroting by \textit{Chronos} under ablations of the high entropy ``diffuse" heads, Fig. \ref{fig:chronos_failure_ablate_sharp_heads} shows how the structure of the predictions falls apart after ablating even a single low entropy ``sharp" head. Note that we present several short-term forecasts, but roll out the predictions to 512 timepoints; we do this for the sake of visibility of the generated median forecasts, which are marked in \textcolor{RoyalBlue}{blue}, with the 20 probabilistic samples from \textit{Chronos} marked using lighter colored lines. 

\textbf{Critical Importance of the Sharp (Low Entropy) Heads}
\begin{figure}[htbp]
    \centering
    \includegraphics[width=1.0\linewidth]{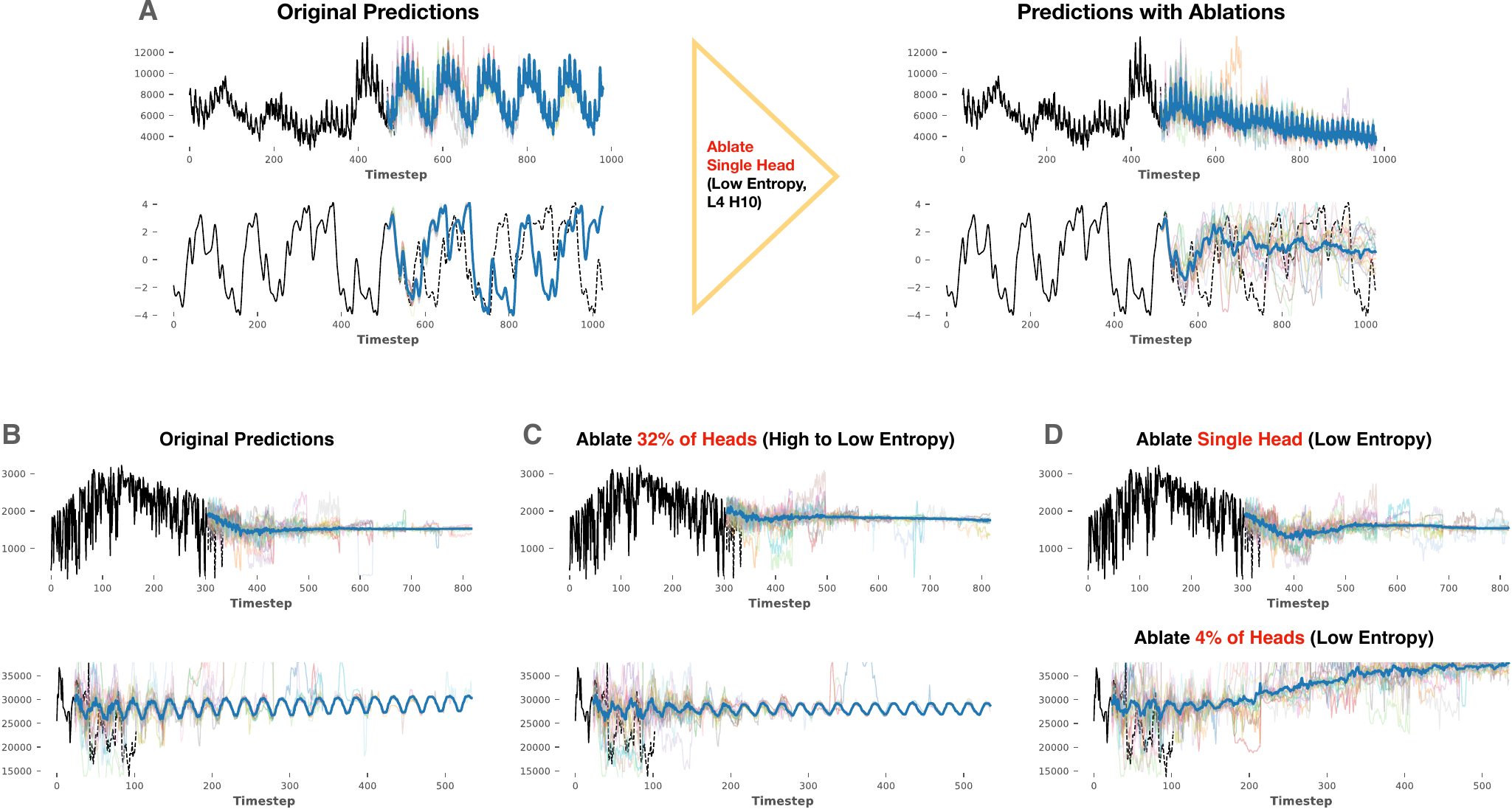}
    \caption{\textbf{Sharp Heads are crucial for context parroting and seasonality}. \textbf{(A)} shows how the structure of the predictions falls apart after ablating even a single low entropy (sharp) head. The forecast tasks are \texttt{m4\_monthly} and the \texttt{Thomas} system from \texttt{GIFT-Eval} and \texttt{dysts} respectively. \textbf{(B)} presents example forecasts on \texttt{solar/D} (Top Row) and \texttt{ett2/W} (bottom row) from \texttt{GIFT-Eval}, for which \textit{Chronos} exhibits degenerate behavior, notably mean regression (Top Row). \textbf{(C)} shows how even these degenerate forecasts are preserved by the ablations of the highest entropy (diffuse) heads, in the same manner as shown in Fig. \ref{fig:head_sharpness}. Meanwhile, \textbf{(D)} shows how ablating a single sharp head preserves the mean regression (Top Row of D) and that ablating only a few (six) sharp heads, i.e. $4\%$ of the total, causes the predictions to lose their structure on rollout (Bottom Row of D).}
    \label{fig:chronos_failure_ablate_sharp_heads}
\end{figure}

\subsection{Attention Rollout Computation} \label{subsection:attention_rollout_computation}
To compute the attention rollouts for Chronos as seen in Figures \ref{fig:head_sharpness}A and \ref{fig:head_sharpness}B, we stored the attention scores for each cross attention head as we forecast a time series for $T$ timesteps given a context of length $C$. Then, for the $t$-th prediction step, we look at how much the token at position $t-1$ in the decoder (for the first prediction step this will be the DECODER\_START token) attend to each token in the context, which will be a vector $a^{t} \in \mathbb{R}^{C}$ where each element $a^t_{i}$ (the amount token $t-1$ in the decoder attends to token $i$ in the encoder) is such that $a^t_{i} \in [0,1]$. Then, as we roll out for a total of $T$ timesteps, we concatenate
$$a^{1:T} = [a^1 \| a^2 \| \cdots \|a^T],$$
where $a^{1:T} \in \mathbb{R}^{C \times T}$ is our attention rollout.

\subsection{Empirical Observations in Patched Models}
\label{app:patchedmodels}

\begin{figure}[H]
    \centering
    \includegraphics[width=0.65\linewidth]{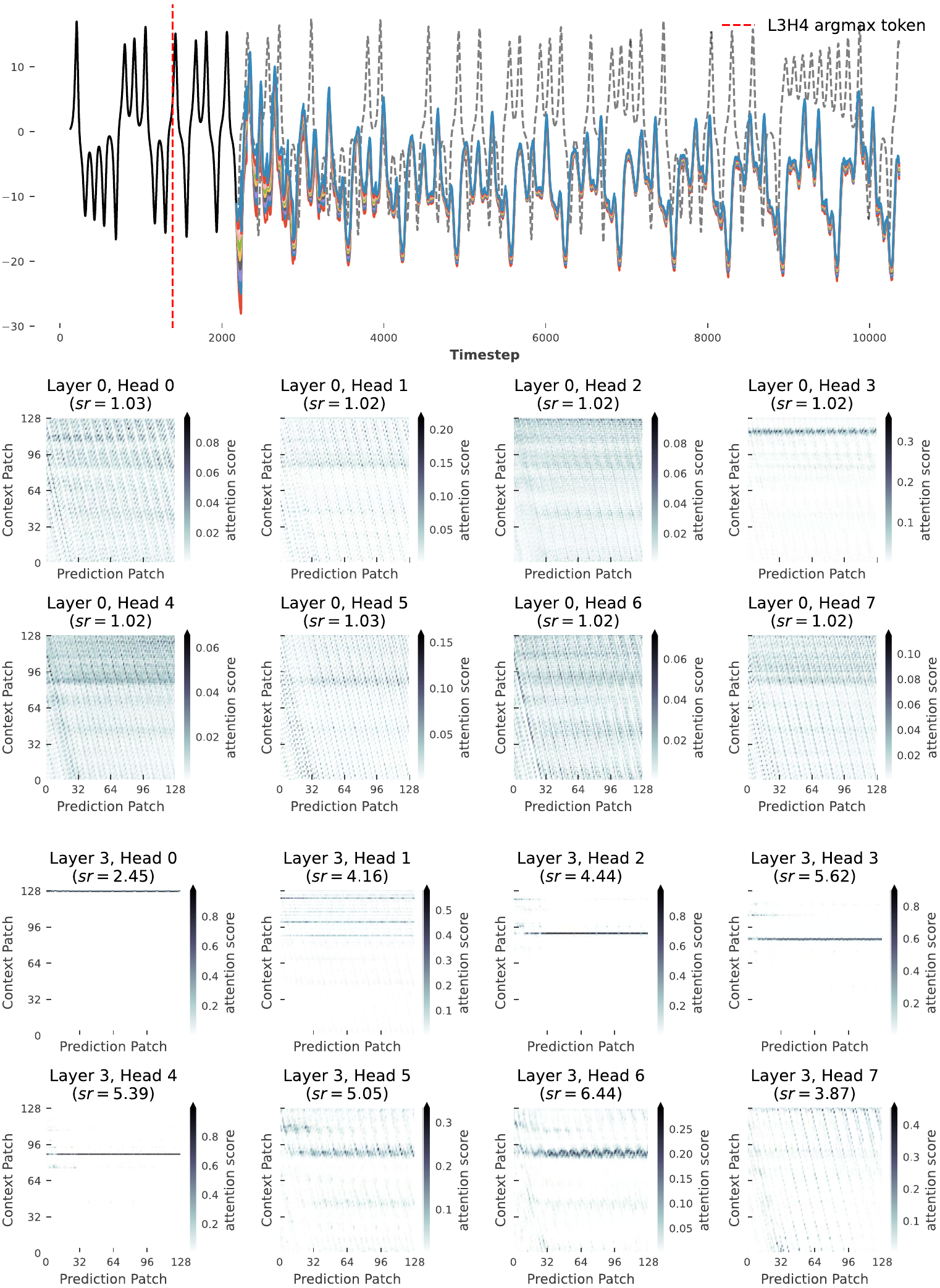}
    \caption{\textbf{Chronos Bolt context parroting on a context window for the Lorenz-63 system}. (A) Length 2048 Lorenz context (solid black), and length 8192 groundtruth (dashed grey) and forecast (colors represent different quantile forecasts). Also shown is the timepoint corresponding to the argmax context patch in the attention matrix in layer 3, head 4 (dashed red). (B) The cross attention matrices across heads for layer 0 with stable ranks annotated for each head. (C) Cross attention matrices across heads for layer 3. See Appendix \ref{app:patchedmodels} for details regarding \textit{Chronos Bolt} and the attention rollout.}
    \label{fig:bolt_parrot}
\end{figure}

We visualize the cross attention matrices over encoder and decoder patches in the \textit{small} checkpoint (48M parameters) of Chronos Bolt. We first note a few important details about the Chronos Bolt (small) architecture:
\begin{enumerate}
    \item The max context length is 2048 which is patched (tokenized) into 128 length 16 patches.
    \item A single register token gets added to the end of the context window for the encoder forward pass.
    \item The decoder output is a single length 64 forecast patch.
\end{enumerate}
Given a length 2048 context window from a randomly sampled Lorenz trajectory, we autoregressively generate a length 8192 quantile forecast and take the median as the point forecast. With patching this configuration allows us to get a $128 \times (128+1)$ cross-attention matrix for the patches. We visualize these matrices in Fig. \ref{fig:bolt_parrot}. 

Chronos Bolt (small) struggles to faithfully parrot motifs from the context, but confidently (small quantile spread) parrots it's own output motif. There also appears to be a concentration of mass around the ~85-th patch in the context ($\approx$ timepoint 1360) in several of the attention heads. Marking this timepoint in the context (Fig. \ref{fig:bolt_parrot}A) suggests that it indicates the start of the context motif Chronos Bolt is \textit{attempting} to parrot. These heads act as selector heads for a fixed point in the context.

Remarkably, compared to \textit{Chronos}, Chronos Bolt exhibits much more of these sharp selector heads, while the shift-like parroting heads are comparatively more diffuse. The selector heads either concentrate mass on a single context timepoint, the register token (attention sink), or a combination of these cases. Interestingly, in layer 3, the selector heads are ablateable according to stable rank potentially due to redundancy, whereas the selector head for the register token (head 0) and the parroting head (head 7) are the least ablateable in contrast (Fig. \ref{fig:bolt_parrot}C).

%% file: sections/appendix/appendix_entropic_rank.tex
\section{Metrics for Comparing Activation Vectors} \label{section:entropic_rank_discussion}

\textbf{Multi-head attention and individual head decomposition:} Let $h \in \mathbb{R}^{T \times d_\text{model}}$ be a residual stream vector at some layer, let head $i \in [H]$ be parametrized by $W_Q^i,W_K^i \in \mathbb{R}^{d_{\text{model}}\times d_k}$,
$W_V^i \in \mathbb{R}^{d_{\text{model}}\times d_v}$, and define
$$
Q^i = hW_Q^i \in \mathbb{R}^{T\times d_k},\quad
K^i = hW_K^i \in \mathbb{R}^{T\times d_k},\quad
V^i = hW_V^i \in \mathbb{R}^{T\times d_v}.
$$

Then, the row-wise attention score matrix for head $i$ is expressed as
$$
A^i = \mathrm{Softmax}\!\left(\frac{Q^i (K^i)^\top}{\sqrt{d_k}}\right)\in \mathbb{R}^{T\times T}.
$$
Finally, multi-head attention (MHA), $\tilde{h}$, is computed from $A^{[H]}$ as
$$
O^i = A^i V^i \in \mathbb{R}^{T\times d_v}
$$
$$
\tilde{h} = \Big[O^1 \;\|\; O^2 \;\|\; \cdots \;\|\; O^n\Big]\, W_O
\;=\; \sum_{i=1}^n O^i W_O^i \in \mathbb{R}^{T\times d_{\text{model}}},
$$
where $W_O \in \mathbb{R}^{(n d_v)\times d_{\text{model}}}$ is the output projection parameter shared among all heads in the layer, with block rows
$W_O^i \in \mathbb{R}^{d_v\times d_{\text{model}}}$. To isolate the individual contribution from the head $i$ to $\tilde{h}$, we simply do
$$\tilde{h}_i := O^iW^i_O \in \mathbb{R}^{T \times d_\text{model}}$$

\textbf{Per layer entropic rank of head outputs:}

For a layer with $H$ heads. For a given sample $s$ and timestep $t$, let
$$
x_i \;:=\; \tilde{h}_i^{(s,t)} \in \mathbb{R}^{d_{\text{model}}},\qquad i\in[H],
$$
be the head output vectors at that $(s,t)$ location.

\emph{Step 1: normalize head vectors.}
$$
u_i \;=\; \frac{x_i}{\|x_i\|_2} \in \mathbb{R}^{d_{\text{model}}}.
$$

\emph{Step 2: form the Gram matrix of cosine similarities.}
Let $U=[u_1^\top;\dots;u_H^\top]\in\mathbb{R}^{H\times d_{\text{model}}}$. Define
$$
G \;=\; UU^\top \in \mathbb{R}^{H\times H},\qquad
G_{ij} = \langle u_i,u_j\rangle.
$$
This matrix is symmetric positive semidefinite, with eigenvalues $\lambda_1,\dots,\lambda_H\ge 0$.

\emph{Step 3: convert spectrum into a probability distribution.} This can be easily done by using the singular values of $G$:
$$
w_k = \sqrt{\lambda_k},\qquad
p_k \;=\; \frac{w_k}{\sum_{j=1}^n w_j},\qquad k\in[H].
$$

\emph{Step 4: compute Shannon entropy and exponentiate.}
$$
H \;=\; -\sum_{k=1}^H p_k \log(\max(p_k,\varepsilon)),\qquad
r \;=\; \exp(H).
$$
We call $r$ the \emph{entropic rank} at that $(s,t)$ location.

\textbf{Averaging across samples and time.}
For each layer, the implementation computes $H^{(s,t)}$ for all samples and timesteps, averages entropy,
\[
\bar{H} \;=\; \mathbb{E}_{s,t}\big[H^{(s,t)}\big],
\]
and finally reports the layer entropic rank as
\[
R \;=\; \exp(\bar{H}).
\]

\textbf{Interpretation and bounds.}
The purpose of the entropic rank measure is to quantify how many heads are spanning diverse directions from each other. In high dimensional vector spaces, this can be interpreted as the number of heads that are nearly orthogonal with respect to each other.

By construction, $R\in[1,H]$. If all heads outputs are mutually orthogonal, then $R=H$. Having a single head output which is highly collinear with another head output while all others are mutually orthogonal would decrease $R$ by one, meaning $R \approx H-1$. Finally, if all heads are approximately collinear with respect to each other, then $R \approx 1$ because there would only be one significant direction.

%% file: sections/appendix/appendix_additional_measurements_residual_stream.tex
\section{Additional Measurements on the Residual Stream} \label{section:more_residual_stream_measurements}

\begin{figure}[htbp]
    \centering
    \includegraphics[width=1.0\linewidth]{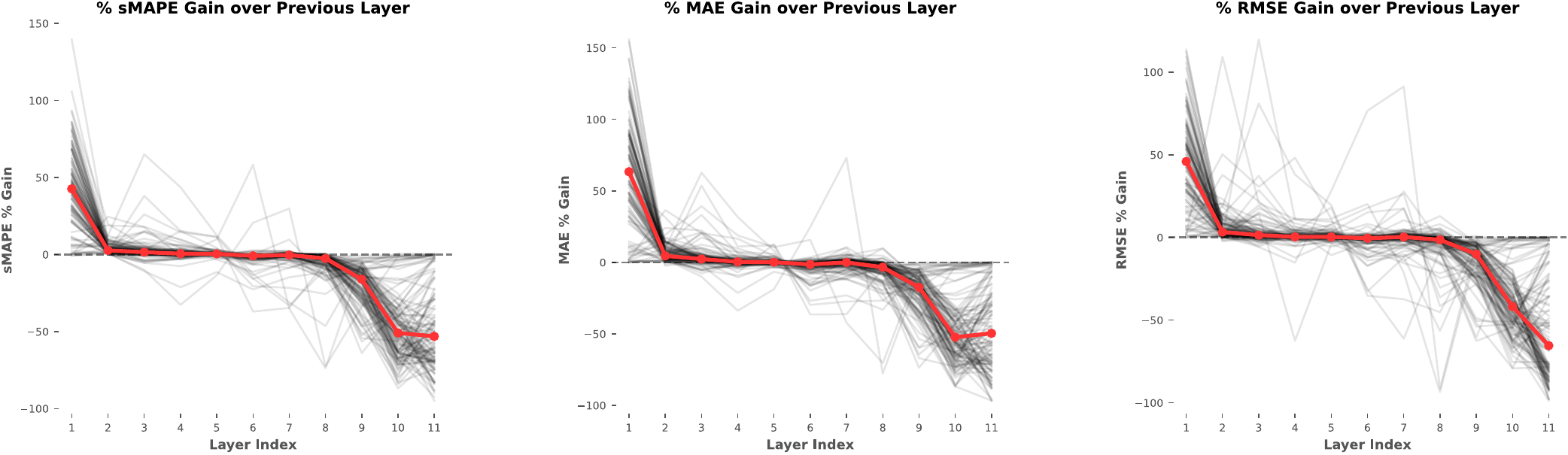}
    \caption{\textbf{Middle layers of Chronos have minimal impact on performance.} We output-transform the embedded tokens after each layer in the \textit{Chronos} decoder, in order to collect the forecasts that would result if we cut off the model after each layer. Across 117 distinct context windows from \texttt{dysts}, we present the percentage gain (lower is better) in several metrics. Each black line corresponds to a context window, and the red lines mark the medians of those. The middle layers consistently show no improvement on the metrics.}
    \label{fig:metrics_improvement_by_layer}
\end{figure}

\begin{figure}[H]
    \centering
    \includegraphics[width=1.0\linewidth]{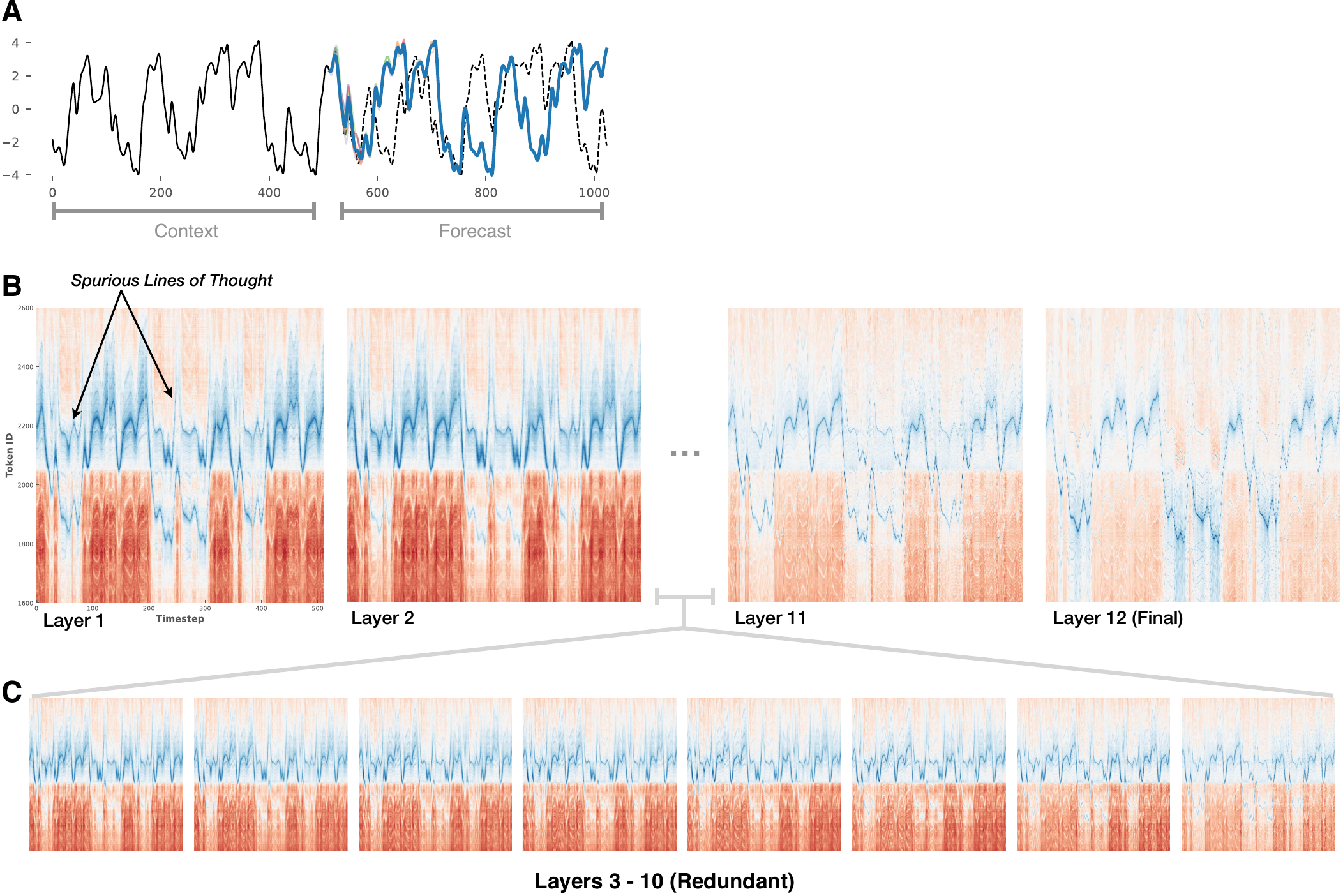}
    \caption{\textbf{Middle layers are highly redundant in their contribution to the residual stream:} (c.f. Fig. \ref{fig:residual_stream_logit_maps}) For an example forecasting task \textbf{(A)}, we investigate the logit maps produced by direct logit attribution (DLA) on the residual stream \textbf{(B)} after each layer in the \textit{Chronos} decoder. As seen in \textbf{(C)}, the middle layers qualitatively perform very similar updates.}
    \label{fig:read_stream_across_layers_Thomas}
\end{figure}

%% file: sections/appendix/appendix_induction_heads.tex
\section{Induction Heads} \label{section:induction_heads_discussion}

\begin{figure}[htbp]
    \centering
    \includegraphics[width=0.9\linewidth]{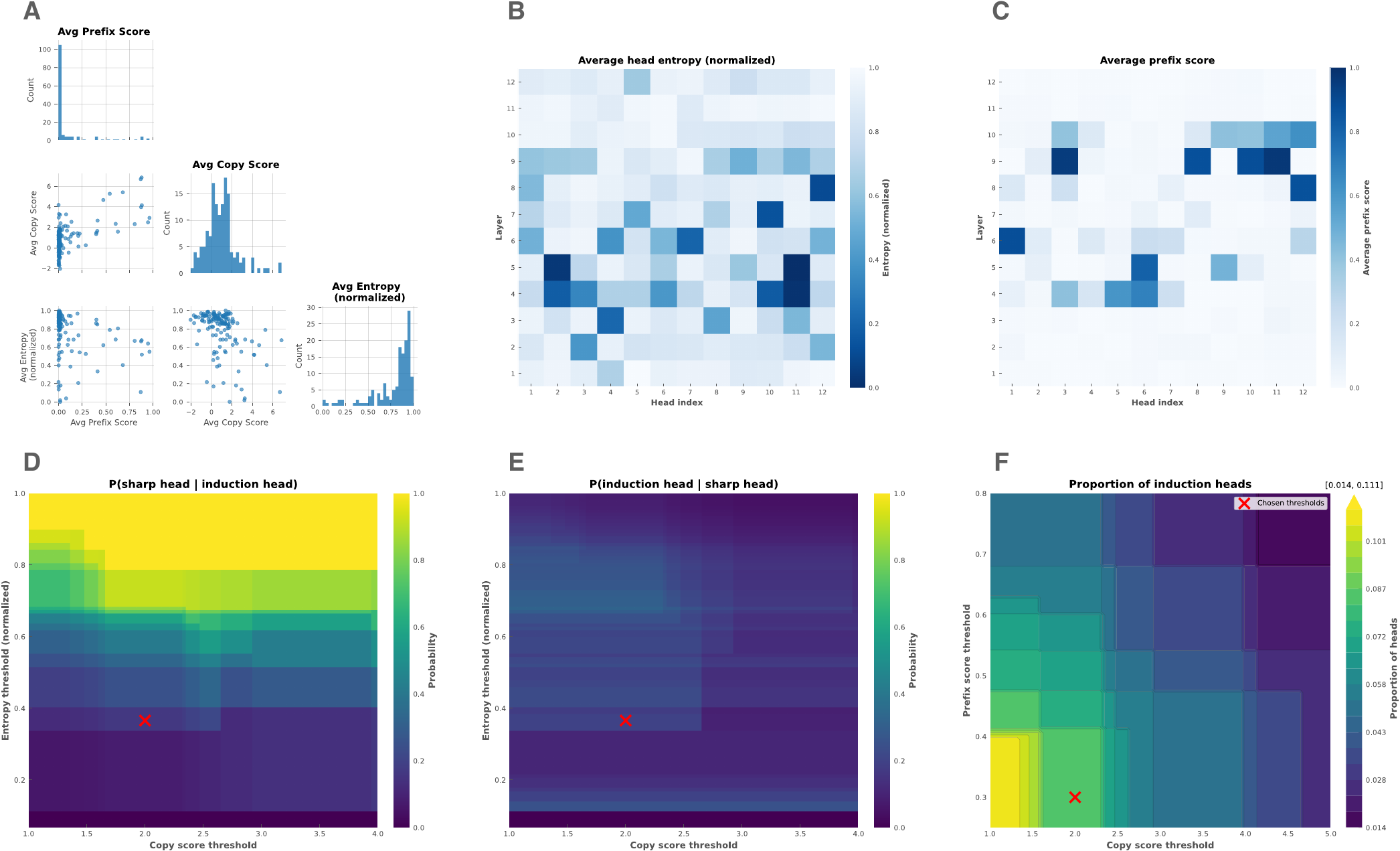}
    \caption{\textbf{Sharp heads and traditional induction heads have some overlap, but their . (A)} Shows a slight positive correlation between the average prefix matching score and the average copy scores of heads, meaning that heads which tend to attend to prefix tokens well in the RRT also slightly tend to be better at copying the correct token over in the sequence. Additionally, we also see slight negative correlations between the prefix score and the average entropy as well as the average copy score and the entropy. \textbf{(B)} Shows the average entropy of each head's attention scores, with lower entropy heads being colored darker. \textbf{(C)} Is similar to A, except we are now visualizing the prefix matching score for each head. Note the similarity between B and C. Heatmap \textbf{(D)} shows the proportion (displayed as probability) of induction heads which are sharp heads and \textbf{(E)} shows the proportion of sharp heads that are induction heads, sweeping over a range of threshold values of entropy, prefix matching score, and copy scores to determine if heads are sharp and/or induction heads.}
    \label{fig:induction_heads_combined}
\end{figure}

\begin{table}[t]
\centering
\begin{tabular}{lrrr}
\hline
 & \textbf{Sharp $S$} & \textbf{Not sharp $\neg S$} & \textbf{Total} \\
\hline
\textbf{Induction $I$}          & 0.0139 & 0.0694 & 0.0833 \\
\textbf{Not induction $\neg I$} & 0.0556 & 0.8611 & 0.9167 \\
\hline
\textbf{Total}                  & 0.0694 & 0.9306 & 1.0000 \\
\hline
\end{tabular}

\vspace{0.6em}

\begin{tabular}{lr}
\hline
\textbf{Overlap ratios} & \textbf{Value} \\
\hline
$P(S \mid I)$ & 0.1667 \\
$P(I \mid S)$ & 0.2000 \\
\hline
\end{tabular}

\caption{Overlap between induction heads ($I$) and sharp heads ($S$), with conditional overlap ratios.}
\label{tab:induction_vs_sharp_heads}
\end{table}

\textbf{Induction Heads:} We defined induction heads as heads which perform prefix matching and copying on a sequence of repeated random tokens, and we use the repeated random tokens (RRT) test to evaluate both criteria \cite{olsson2022context}. By comparing induction heads and sharp heads in the Chronos model, we find that 16.7\% of induction heads are sharp heads and also about 20.0\% of sharp heads are also induction heads. Though, by relaxing the constraints we set for a head to be classified as an induction head or a sharp head, we see that 50-60\% induction heads could potentially be sharp heads, but the proportion of sharp heads which are also induction heads continues to stay at around 20\%. Thus, we feel confident that despite a significant amount of induction heads being able to act like sharp heads, there are still many sharp heads which are distinctly not induction heads in the traditional sense as defined by the RRT test.

Because Chronos is the only model in our study which uses a uniform binning tokenizer and isn't patch based, in which case the RRT test translates directly from the language domain, we only conducted the RRT test on Chronos.

The repeated random tokens test we used a sequence of 32 randomly sampled tokens from token ids 1911-2187 (corresponding to frequently used tokens that are within one standard deviation of the context according to the internals of the Chronos tokenizer \cite{ansari2024chronos}) and repeated them 4 times, similar to the procedure followed in \cite{olsson2022context}. Then, due to the encoder-decoder nature of the model, we appended an EOS token to the repeated sequences and input them to the encoder, and finally we input a DECODER\_START token and the first token of the sequence into the decoder to initiate the auto-regressive rollout.

Then, we computed the prefix matching score of a cross-attention head as the average amount it attended to the most recent previous instance (in the encoder's context) of the current token, the copy score as the standard deviations above the mean which the model attributes the correct token to with respect to all tokens in the vocabulary, and the head entropy as the average entropy across time for how much each current token in the rollout attends to each token in the context. Essentially, we computed the entropy for each vertical slice in the plots shown in Figures \ref{fig:head_sharpness}A and \ref{fig:head_sharpness}B. We then normalized the average entropy values of all heads to range between 0 and 1. We found that heads with high prefix scores loosely tend to have higher copy scores as well as a lower average entropy in Figure \ref{fig:induction_heads_combined}A. Furthermore, Figures \ref{fig:induction_heads_combined}B and \ref{fig:induction_heads_combined}C visually show that some heads with a low average entropy tend to also have high prefix matching scores.

We see the distributions and pairwise relationships between the average prefix score, the copy score, and the average normalized entropy in Figure \ref{fig:induction_heads_combined}A.

We found that out of the $H=144$ heads in Chronos base 200M, 8.3\% of them are able to perform prefix matching as well as copying (i.e. they are induction heads). The thresholds used to classify heads as being capable of prefix matching is a prefix matching score of at least 0.3 (as found in \cite{crosbie-shutova-2025-induction}), and a head was classified as being able to copy if the head's logit attribution increased the logits of the correct token at least two standard deviations above the mean across all tokens. The maximum threshold entropy of a sharp head was set to be approximately 0.38 when normalized. The classification of a head being an induction head or a sharp head is sensitive to the choice of thresholds, and we explored how varying the thresholds would change the number of induction heads and sharp heads we observe in Chronos in Figures \ref{fig:induction_heads_combined}D, \ref{fig:induction_heads_combined}E, and \ref{fig:induction_heads_combined}F. For Figures, \ref{fig:induction_heads_combined}D and \ref{fig:induction_heads_combined}E, we fixed a prefix matching threshold of 0.3.

Despite sharp heads making up only 6.9\% of all heads and induction heads only making up 8.3\% of all heads, they had better-than-random overlap with 16.7\% of induction heads being sharp heads and 20.0\% of sharp heads being induction heads. So, clearly they are not fully independent of each other. Furthermore, as we observe in Figures \ref{fig:induction_heads_combined}D and \ref{fig:induction_heads_combined}E, increasing the maximum threshold for a sharp head can dramatically increase the proportion, up to around 50-60\%, of induction heads which are also sharp (while still maintaining a reasonable threshold), but we observe an approximately constant 20\% of sharp heads being induction heads regardless of how we vary the entropy and copy score thresholds. This suggests that a significant proportion of induction heads act like sharp heads in some capacity, but around 80\% of sharp heads cannot be classified as induction heads. The marginal and conditional proportions of sharp and induction heads are displayed in Table \ref{tab:induction_vs_sharp_heads}.

Still, the sharp heads which are not induction heads as defined by the RRT test are clearly able to match recurring motifs in the context and are responsible for copying them as observed in Figure \ref{fig:head_sharpness}. Thus, it is likely that TSFMs could learn to perform induction through a different paradigm than what is measured by the RRT test. 